\documentclass{article}

\usepackage[preprint]{neurips_2026}

\usepackage[utf8]{inputenc} 
\usepackage[T1]{fontenc}    
\usepackage{hyperref}       
\usepackage{url}            
\usepackage{booktabs}       
\usepackage{amsfonts}       
\usepackage{nicefrac}       
\usepackage{microtype}      
\usepackage{xcolor}         

\usepackage[ruled,vlined]{algorithm2e}
\usepackage{graphicx}
\usepackage{wrapfig}
\usepackage{subcaption}
\usepackage{amsmath, amssymb, amsthm}
\newtheorem{theorem}{Theorem}[section]
\newtheorem{lemma}[theorem]{Lemma}

\newtheorem{remark}{Remark}
\newtheorem{definition}{Def.}[section]

\numberwithin{equation}{subsection}

\DeclareUnicodeCharacter{2212}{-}

\newcommand{\indep}{\mathrel{\perp\!\!\!\perp}}

\title{Data-Driven and Cycle-Agnostic Covariate Selection for Nonparametric Causal Effect Estimation}

\author{%
  Ana Let\'{i}cia Garcez Vicente\thanks{Equal contribution} \\
  Leiden University\\
  Netherlands\\
  \texttt{a.l.garcez.vicente@liacs.leidenuniv.nl} \\
  \And
  Gijs van Seeventer\footnotemark[1] \\
  Leiden University\\
  Netherlands\\
  \texttt{g.g.van.seeventer@liacs.leidenuniv.nl} \\
  \And
  Saber Salehkaleybar \\
  Leiden University\\
  Netherlands\\
  \texttt{s.salehkaleybar@liacs.leidenuniv.nl} \\
}

\begin{document}

\maketitle

\begin{abstract}

  Estimating causal effects from observational data requires identifying valid adjustment sets. This task is especially challenging in realistic settings where latent confounding and feedback loops are present.
  Existing approaches typically assume acyclicity or rely on global causal structure learning, limiting applicability and computational efficiency. 
  In this work, we study a local, data-driven method for covariate selection based on conditional independence information under a pre-treatment assumption. While this method is known to be sound and complete in acyclic causal models, its validity in the presence of cycles has remained unclear. Our main contribution is to show that these guarantees extend to cyclic causal models. In particular, our result relies on the invariance of conditional independence assertions under $\sigma$-acyclification. These findings establish a unified, cycle-agnostic perspective on covariate selection, which allows for causal effect estimation using any consistent estimator. Therefore, showing that the method applies across cyclic and acyclic settings without modification. 
  Empirically, we validate this on extensive synthetic data, showing reliable performance in cyclic causal models.
\end{abstract}

\section{Introduction}

Estimating causal effects between variables is a central objective across a wide range of scientific disciplines. In epidemiology and pharmacology, for example, understanding the effect of a treatment on an outcome is critical for public health and drug development \cite{hernan2006estimating}. In economics, causal questions arise when evaluating policy interventions or market mechanisms \cite{angrist2009mostly}. The standard for such tasks is the randomized controlled trial (RCT), where treatments are randomly assigned to eliminate confounding bias \cite{akobeng2005understanding}. However, RCTs are often infeasible due to ethical constraints, cost, or even the impossibility of manipulating certain variables \cite{rosenbaum2017observation}. 

As a result, there has been extensive interest in estimating causal effects from observational (non-experimental) data. A central challenge in this setting is confounding, i.e., spurious associations induced by common causes of treatment and outcome. A common approach to address confounding is covariate adjustment, where one conditions on a suitable set of variables, an adjustment set, to identify the causal effect \cite{bareinboim2016causal}. Many such approaches are formulated in a non-parametric framework, avoiding restrictive assumptions on functional forms and enabling broad applicability across data-generating processes. For instance, it allows the use of machine-learning-based estimators, which are popular in practice \cite{Talbot2025}.

Despite substantial progress \cite{li2024local,cheng2022local}, 
existing literature on adjustment set identification assumes that the underlying causal graph is acyclic, 
and does not extend to cyclic settings \cite{peters2017elements}. This assumption excludes a broad class of real-world systems that exhibit feedback loops and cyclic dependencies that are considered relevant in the causal inference literature \cite{pearl95a,dai2026distributional}. Such cycles naturally arise in many domains, including biological regulatory networks (e.g., gene–protein feedback), economic systems (e.g., supply–demand interactions), and ecological systems, where variables influence each other over time.

\begin{wrapfigure}[17]{r}{0.34\linewidth}
    \includegraphics[width=\linewidth]{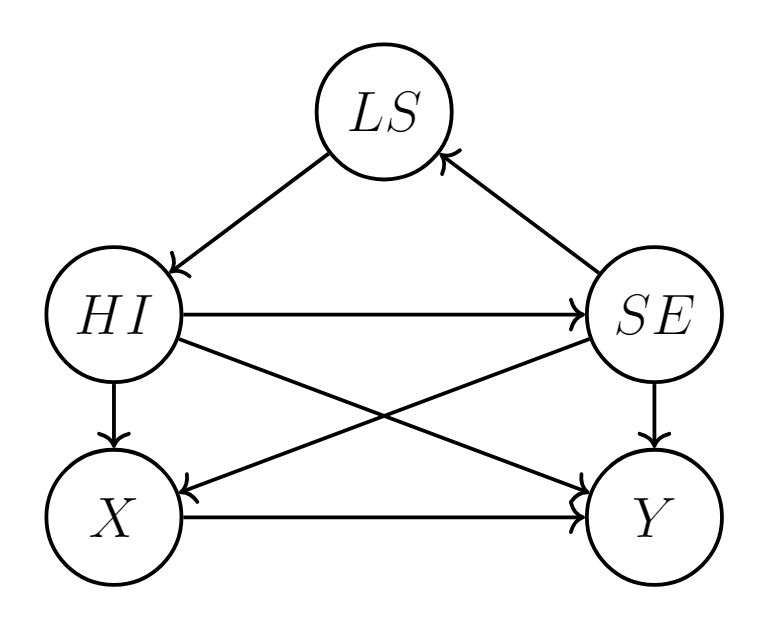}
    \caption{A causal graph of a health care setting. $LS$ represents lifestyle, $HI$ health indicators, $SE$ socioeconomic status, $X$ the treatment and $Y$ the outcome.}
    \label{fig:health_care_sketch}
\end{wrapfigure}

To provide a more detailed example, consider a health care setting. In cancer studies, baseline covariates often include socioeconomic status, lifestyle behaviors, and health indicators \cite{wild2020world}. These variables are not independent and may form feedback structures before treatment assignment. For example, consider the causal graph in Figure~\ref{fig:health_care_sketch}. Socioeconomic status can influence lifestyle. Lifestyle can shape health indicators, which may in turn influence socioeconomic status, such as employment. This creates feedback among covariates. At the same time, these covariates may affect both treatment assignment, such as access to screening or choice of therapy, and cancer recovery outcomes. Thus, the covariate set itself may form a cyclic structure that must be accounted for when selecting valid adjustment sets.

This mismatch between modeling assumptions and real-world structure exposes a key limitation of existing methods: they do not extend to cyclic settings and therefore do not provide guarantees for covariate adjustment in the presence of feedback loops. In this paper, we address this gap by studying covariate selection in causal models that allow for cyclic dependencies. Our main contributions are:
(i) we show that a local covariate selection method that allows for causal effect estimation \cite{li2024local}, defined in a nonparametric setting under a pre-treatment assumption, remains valid in both acyclic and cyclic systems without requiring prior knowledge of the underlying structure;
(ii) we extend theoretical guarantees of soundness and completeness for this method from acyclic to cyclic causal structures;
(iii) we empirically validate these results through extensive simulation on synthetic data, comparing cyclic and acyclic settings under varying structural conditions. 


\section{Preliminaries and problem formulation} \label{sec:preliminaries}

We begin by introducing structural causal models and their graphical representation. A structural causal model (SCM), defined as a tuple $\mathcal{M} = (\mathbf{V}, \mathbf{U}, F, P(\mathbf{U}))$, where $\mathbf{V}$ is the set of endogenous variables, $\mathbf{U}$ is the set of exogenous variables with joint distribution $P(\mathbf{U})$, and $F = \{f_i\}_{i=1}^{|\mathbf{V}|}$ is a collection of functions such that each endogenous variable satisfies a structural equation $X_i = f_i(Pa(X_i), U_i)$. We restrict attention to simple SCMs \cite{bongers2021foundations}, meaning that for every subset of variables, the corresponding subsystem of structural equations admits a unique solution. This ensures well-defined observational and interventional distributions even in the presence of cycles, and guarantees that standard graphical and probabilistic semantics extend to cyclic settings. It is noteworthy to mention that all acyclic SCMs are simple.

An SCM induces a directed graph (DG), which 
may contain directed cycles. Let $G=(\mathbf{V},\mathbf{E})$ be the directed graph associated with $\mathcal{M}$. We use the same symbol $\mathbf{V}$ to denote both the
set of endogenous variables in the SCM and the corresponding node set in
$G$, with each node representing one endogenous variable. The edge set
$\mathbf{E}\subseteq \mathbf{V}\times\mathbf{V}$ contains a directed edge
$X_i\to X_j$ whenever $X_i$ appears as an argument in the structural
function for $X_j$.
We partition the nodes as $\mathbf{V} =  \mathbf{O} \cup \mathbf{L}$,  where $\mathbf{O}$ denotes the set of observed variables, and $\mathbf{L}$ the set of latent (unobserved endogenous) variables. Moreover, we assume that the treatment and outcome are observable, i.e., $X,Y\in \mathbf{O}$. A DG is said to be acyclic if it contains no directed cycles, and it is called a directed acyclic graph (DAG); otherwise, it is cyclic.

A path between two nodes $X$ and $Y$ is a sequence $(X = Z_0, Z_1, \dots, Z_k = Y)$ such that each consecutive pair of nodes is adjacent in $G$, and it is called directed if all edges are oriented consistently along the sequence. For any node $X \in \mathbf{V}$, we denote by $Pa(X) = \{Z \in \mathbf{V} \mid Z \to X \in \mathbf{E}\}$ the set of its parents, by $\mathrm{Ch}(X) = \{Z \in \mathbf{V} \mid X \to Z \in \mathbf{E}\}$ the set of its children, and define its ancestors $\text{An}(X)$ and descendants $\mathrm{De}(X)$ as the sets of nodes that admit a directed path to $X$ and from $X$, respectively. A strongly connected component (SCC) is a maximal subset of nodes $\mathbf{C} \subseteq \mathbf{V}$ such that for every $X,Y \in  \mathbf{C}$ there exist directed paths from $X$ to $Y$ and from $Y$ to $X$.

To read off conditional independence from a possibly cyclic graph, we adopt $\sigma$-separation \cite{forre2017markov}. 

\begin{definition}[$\sigma$-separation]\label{def:sigma-sep}
Suppose $G = (\mathbf{V}, \mathbf{E})$ is a DG, $X$ and $Y$ are two distinct variables in $\mathbf{V}$, and $\mathbf{S} \subseteq \mathbf{V} \setminus \{X, Y\}$. A path $\mathcal{P} = (X = Z_0, Z_1, \cdots, Z_k, Z_{k+1} = Y)$ between $X$ and $Y$ in $G$ is $\sigma$-blocked by $\mathbf{S}$ if there exists $1 \leq i \leq k$ such that:
\begin{enumerate}
    \item $Z_i$ is a collider on $\mathcal{P}$ and $Z_i \notin \text{An}_{G}(\mathbf{S})$, or
    \item \begin{enumerate}
        \item $Z_i$ is not a collider on $\mathcal{P}$, $Z_i \in \mathbf{S}$, and
        \item either $Z_i \to Z_{i+1}$ and $Z_{i+1} \notin \mathrm{SCC}_{G}(Z_i)$, or $Z_{i-1} \leftarrow Z_i$ and $Z_{i-1} \notin \mathrm{SCC}_{G}(Z_i)$.
    \end{enumerate} 
\end{enumerate}

\end{definition}
Since condition $2. (b)$ is always satisfied in DAGs, this definition reduces to standard d-separation \cite{pearl2009causality} when $G$ is acyclic. We write $X \indep_{\sigma/d} Y \mid \mathbf{S}$ to denote $\sigma/d$-separation.
We write $X \indep Y \mid \mathbf{S}$ to denote conditional independence in the joint distribution $P(\mathbf{V})$, and assume the Markov and faithfulness properties in the following generalized form: for all disjoint sets $\mathbf{A},\mathbf{B},\mathbf{S} \subseteq \mathbf{V}$, $\mathbf{A} \indep_\sigma \mathbf{B} \mid \mathbf{S}$ if and only if $\mathbf{A} \indep \mathbf{B} \mid \mathbf{S}$. In cyclic graphs this corresponds to $\sigma$-faithfulness, while in DAGs it reduces to standard faithfulness.



The notion of a Markov blanket depends on the underlying graph and separation criterion.
Given a graph $G$, for a node $X \in \mathbf{O}$, the observed Markov Blanket, denoted by $\mathrm{MB}_\sigma^G(X)$, is the minimal set
$\mathbf{S} \subseteq \mathbf{O} \setminus \{X\}$ such that $X \indep_\sigma \mathbf{O} \setminus (\mathbf{S} \cup X) \mid \mathbf{S}$
. When $G$ is a DAG, this reduces to the standard (d-separation-based) Markov blanket, denoted by $\mathrm{MB}_d^G(X)$.


Throughout the paper, we impose the following assumptions: (i) We consider the pre-treatment assumption, which amounts to a 
treatment--outcome structure in which the outcome $Y$ has no descendants, i.e., $\mathrm{De}(Y)=\emptyset$, and the treatment $X$ has either no children or a single child equal to the outcome, i.e., $\mathrm{Ch}(X)=\emptyset$ or $\mathrm{Ch}(X)=\{Y\}$.
(ii) We allow for latent confounding represented by unobserved endogenous variables $\mathbf{L}$, and assume that 
$\sigma$-faithfulness holds for the induced observational distribution restricted to $\mathbf{O}$.

\begin{remark}
    The pre-treatment assumption is common in the literature, see for instance Remark 1 in \cite{li2024local}. In fact,  it is realistic in many application areas, such as economics and epidemiology, where covariates are often collected before treatment assignment and before the outcome is observed \citep{hill2011bayesian,imbens2015causal,wager2018estimation}. When not using the valid adjustment set, this can introduce bias \cite{Helske2021,Thoemmes_2015}.
\end{remark} 

\textbf{Example:}  In a health care setting (Figure~\ref{fig:health_care_sketch}), variables such as socioeconomic status (SE) (e.g., pre-treatment employment), health indicators (HI) (e.g., prior hospitalizations), and lifestyle (LS) (e.g., pre-treatment physical activity) can be modeled as covariates measured solely prior to treatment assignment. Although these variables may influence one another dynamically over time, they are represented here as a pre-treatment snapshot and hence by time ordering not affected by the treatment $X$ or outcome $Y$.

\section{Related works}

\paragraph{Adjustment sets and causal effect estimation} Graphical criteria provide necessary and sufficient conditions for covariate adjustment in causal graphs. In DAGs, the back-door criterion \cite{pearl1993bayesian, pearl2009causality} characterizes when a set of variables $\mathbf{Z}$ is a valid adjustment, by requiring that all non-causal 
paths between treatment $X$ and outcome $Y$ are blocked and $\mathbf{Z}$ does not contain descendants of $X$. This criterion has been generalized to characterize necessary and sufficient valid adjustment in equivalence classes of causal graphs with latent confounding, such as 
maximal ancestral graphs (MAGs) \cite{perkovic2018complete}.

In practice, these criteria assume that the causal graph or the equivalence class is available. When the structure is unknown, a common approach is to first estimate an equivalence class (i.e., global structure) using constraint-based algorithms such as PC \cite{spirtes1991algorithm}, and then perform adjustment. One such method is IDA \cite{maathuis2009estimating}, which enumerates all possible parent sets of treatment with the estimated graph and computes causal effect estimates. Building on this, LV-IDA 
\cite{malinsky2017estimating} generalizes the approach to PAGs to account for latent confounding. In contrast, CE-SAT \cite{hyttinen2015calculus} avoids enumerating all graphs in the equivalence class by encoding the problem as a constraint task.

Complementary approaches aim to identify adjustment sets directly from conditional independence (CI) tests without recovering the equivalence graph first. Assuming that all variables are observed, \cite{pmlr-v300-schubert26a} recently proposed LOAD, an algorithm that finds an optimal valid adjustment set. However, our main interest lies in the setting that allows for latent confounding. In that setting, the EHS algorithm  \cite{entner2013data} applies inference rules over CI to determine whether a causal effect is identifiable and to construct a valid adjustment set when it exists. EHS is sound and complete under the pre-treatment assumption, i.e., considering that covariates are not affected by the treatment. However, it requires an exhaustive search over subsets of variables. More recent methods aim to improve efficiency by exploiting locality. CEELS 
\cite{cheng2022local} uses structural patterns and COSO variables to guide a local search for adjustment sets. 
LSAS 
\cite{li2024local} shows that EHS can be restricted to search over the Markov blanket while remaining sound and complete. This reduction in the search space results in shorter run times and substantially fewer independence tests compared to EHS. On the other hand, LDP relaxes the pre-treatment assumption and introduces sufficient conditions, but is no longer complete \cite{maasch24a}.
Concurrent with our work, the LCS algorithm was proposed by \cite{liu2026local}. LCS relaxes the pre-treatment assumption while being sound and complete. However, it assumes that the underlying causal graph is acyclic.

Overall, existing approaches either rely on global structure learning with high computational cost, or perform local search with some trade-offs between efficiency and completeness, but 
exclusively assume acyclicity.

\paragraph{Markov blankets} Local causal discovery methods aim to identify subsets of variables relevant to a target using CI tests, without recovering the full causal graph. Algorithms such as (Fast)-IAMB \cite{tsamardinos2003algorithms, yaramakala2005speculative}, HITON-MB \cite{aliferis2003hiton} and TC \cite{pellet2008using} 
identify the Markov blanket set 
via CI tests. 
From the perspective of finding valid adjustment sets, the Markov blankets efficiently provide supersets of the relevant variables (treatment and outcome), but do not directly compute valid adjustment sets, since they may include colliders on paths between the treatment and the outcome or descendants of the treatment.

\paragraph{Cyclic causal models} 
Results on covariate adjustment set 
assume acyclicity, where d-separation characterizes conditional independence. However, causal models with cycles require more general notions of separation. \cite{forre2017markov} introduced $\sigma$-separation, a generalization of d-separation for directed cyclic graphs, accounting for strongly connected components dependence. Building on this, \cite{bongers2021foundations} defined the class of simple SCMs and established equivalence notions for cyclic graphs. Furthermore, \cite{mooij2020constraint} introduces the notion of acyclicification, which maps a cyclic graph to an acyclic one while preserving relevant conditional independence relations 
. In addition, \cite{mooij2020constraint} shows that causal discovery algorithms like FCI work for simple cyclic SCMs. These developments provide a foundation for reasoning about conditional independence and equivalence in cyclic systems. However, to the best of our knowledge, they have not been utilized in designing covariate selection algorithms in cyclic models.

\section{Acyclification}\label{sec:acy}

To extend graphical criteria from acyclic to cyclic models, we rely on the notion of $\sigma$-acyclification \cite{mooij2020constraint}, which removes directed cycles while preserving separation properties. 

\begin{definition}[$\sigma$-acyclification]
    Let $G=(\mathbf{V},\mathbf{E})$ be a directed graph, possibly cyclic. An acyclification of $G$ is a directed acyclic graph $G^{\mathrm{acy}}=(\mathbf{V},\mathbf{E}^{\mathrm{acy}})$ defined on the same node set $\mathbf{V}$ and constructed as follows. For any pair of distinct nodes $X, Y \in \mathbf{V}$:
    \begin{enumerate}
        \item if $X \notin \mathrm{SCC}(Y)$, then 
        $X \to Y \in \mathbf{E}^{\mathrm{acy}}$ if and only if $\ \exists\, W \in \mathrm{SCC}(Y) \text{ such that } X \to W \in \mathbf{E}$;
        \item if $X \in \mathrm{SCC}(Y)$
        , then exactly one of the edges $X \to Y$ or $Y \to X$ is included in $\mathbf{E}^{\mathrm{acy}}$ so as to break all directed cycles.
    \end{enumerate}
\end{definition}
Intuitively, all nodes in a SCC are made adjacent in an acyclic manner, while edges between different components are preserved by connecting every node in the source component to every node in the target component. By construction, $G^{\mathrm{acy}}$ is a DAG, and the acyclification is not necessarily unique.

A key property of this construction is that it preserves separation relations. In particular, $\sigma$-separation in $G$ coincides with d-separation in any $\sigma$-acyclification $G^{\mathrm{acy}}$, allowing graphical arguments in cyclic models to be reduced to the acyclic case.

\begin{lemma}[Acyclification transfer of separation \cite{mooij2020constraint}]
\label{lem:acy-transfer-sep}
Let $G=(\mathbf{V},\mathbf{E})$ be a directed graph and let $G^{\mathrm{acy}}$ be any $\sigma$-acyclification of $G$. Then for all distinct $A,B\in \mathbf{V}$ and all $\mathbf{Z}\subseteq \mathbf{V}\setminus\{A,B\}$,
\[
(A \indep_d B \mid \mathbf{Z})_{G^{\mathrm{acy}}}
\quad\iff\quad
(A \indep_\sigma B \mid \mathbf{Z})_G.
\]
\end{lemma}

An immediate consequence of Lemma~\ref{lem:acy-transfer-sep} is that local separation-based objects are invariant under acyclification. In particular, Markov blankets coincide in the cyclic and acyclic representations. 

\begin{lemma}[Markov blanket invariance]
\label{lem:mb-acy}
Let $G$ be a directed graph and $G^{\mathrm{acy}}$ any $\sigma$-acyclification. Then for any node $Y$,
\[
\mathrm{MB}^G_{\sigma}(Y) = \mathrm{MB}^{G^{\mathrm{acy}}}_{d}(Y).
\]
\end{lemma}
The proofs of all lemmas and theorems are given in the appendix. We next show that acyclification can be chosen so as to preserve specific local edge structures.

\begin{lemma}[Preservation of the edge $X \to Y$ under $\sigma$-acyclification]
\label{lem:edge-preserved}
    Let $G=(\mathbf{V},\mathbf{E})$ be a directed graph, possibly cyclic, and let $X,Y \in \mathbf{V}$ be distinct nodes. Assume that $\mathrm{De}(Y)=\emptyset$, and $\mathrm{De}(X) \subseteq \{Y\}$. Then all $\sigma$-acyclification $\hat G^{\mathrm{acy}}=(\mathbf{V},\mathbf{E}^{\mathrm{acy}})$ such that $Y$ remains a sink node (has no children), no node other than $Y$ can be a child of $X$, i.e., $\mathrm{De}(X) \subseteq \{Y\}$, and the edge $X \to Y$ is present in $\mathbf{E}^{\mathrm{acy}}$ if and only if it is present in $\mathbf{E}$.
\end{lemma}

In the proofs of the remaining theorems and lemmas, we repeatedly use the same strategy: map a possibly cyclic directed graph $G$ to (all) acyclifications $G^{\mathrm{acy}}$ using results from this section, apply existing results for DAGs, and then transfer the conclusions back to $G$.

\section{Adjustment set}

A common way to identify the causal effect of $X$ on $Y$ using solely observational data is based on using a valid adjustment set \cite{pearl2009causality,peters2017elements}.

\begin{definition}[Valid adjustment set]
Let \(G=(\mathbf{V},\mathbf{E})\) be a directed graph, and \(\{X\},\{Y\},\mathbf{Z} \subseteq \mathbf{O}\) be pairwise disjoint node sets. The set \(\mathbf{Z}\) is called an \emph{adjustment set} for estimating the causal effect of \(X\) on \(Y\) if, for all probability distributions \(P\) Markov with respect to \(G\), the following holds:
\[
P(Y \mid \operatorname{do}(X)) \;=\; \sum_{z\in \mathbf{Z}} P(Y \mid X, \mathbf{Z}=z)\, P(\mathbf{Z}=z).
\]
\end{definition}

The invariance results in Section~\ref{sec:acy} allow us to reduce valid adjustment questions in cyclic graphs to the acyclic case. For DAGs, there exists a valid adjustment set known as the backdoor adjustment set formulated in terms of $d$-separation \cite{pearl2009causality,peters2017elements}. Replacing $d$-separation by $\sigma$-separation, we get a backdoor adjustment set in possibly cyclic SCMs.

\begin{definition}[Backdoor adjustment set in possibly cyclic SCMs]
\label{def:backdoor-adjustment}
Let $G=(\mathbf{V},\mathbf{E})$ be a (possibly cyclic) directed graph satisfying the assumptions of Section~\ref{sec:preliminaries}, and let $X,Y \in \mathbf{O}$ be treatment and outcome variables. A set $\mathbf{Z} \subseteq \mathbf{O} \setminus \{X,Y\}$ is a backdoor  adjustment set for the effect of $X$ on $Y$ if the following conditions hold:
\begin{enumerate}
    \item $\mathbf{Z} \cap \mathrm{De}(X) = \emptyset$;
    \item $\mathbf{Z}$ $\sigma$-blocks every backdoor path from $X$ to $Y$ (i.e., every path between $X$ and $Y$ whose first edge on the path has an arrowhead into $X$).
\end{enumerate}

\end{definition}




If the SCM is acyclic, this reduces to the standard definition. Therefore, in DAGs, Definition~\ref{def:backdoor-adjustment} of backdoor adjustment provides valid adjustment sets. To show that the definition gives a valid adjustment set in simple cyclic SCMs, we can use results from \cite{forre2020causal}. To use their results, we model interventions by augmenting the graph with an intervention node \cite{pearl2009causality,peters2017elements,forre2020causal}. For a treatment variable $X$, we define $G^{I_X}$ as the graph obtained by adding a new node $I_X$ and a single directed edge $I_X \to X$. The node $I_X$ encodes the external intervention $\operatorname{do}(X)$ and has no parents. With this notation, we can show the equivalence of their formulation and Definition~\ref{def:backdoor-adjustment}.

\begin{lemma}[Equivalence with the intervention-node criterion]
\label{lem:equiv-backdoor}
Let $G$ be a (possibly cyclic) SCM and $\mathbf{Z} \subseteq \mathbf{V} \setminus \{X,Y\}$. Let $G^{I_X}$ denote the graph obtained by adding an intervention node $I_X$ with the single directed edge $I_X \to X$. Then $\mathbf{Z}$ is a valid backdoor adjustment set in the sense of Definition~\ref{def:backdoor-adjustment} if and only if
\begin{enumerate}
    \item $\mathbf{Z} \indep_\sigma I_X$ in $G^{I_X}$, and
    \item $Y \indep_\sigma I_X \mid X,\mathbf{Z}$ in $G^{I_X}$.
\end{enumerate}
\end{lemma}

Using Lemma~\ref{lem:equiv-backdoor}, we can claim that, by Corollary~8.4 from \cite{forre2020causal}, Definition~\ref {def:backdoor-adjustment} is indeed a valid adjustment set for simple cyclic SCMs. Furthermore, with this result, we can show a useful property of backdoor adjustment sets with respect to Markov blankets, which was shown for acyclic SCMs by \cite{li2024local}.

\begin{lemma}[Existence of an adjustment set]
\label{lem:existence-adjustement}
Let \(G=(\mathbf{V},\mathbf{E})\) be a simple (possibly cyclic) directed graph satisfying the assumptions of Section \ref{sec:preliminaries}.
Then there exists a subset \(\mathbf{Z}\subseteq \mathbf{O}\) that is a backdoor adjustment set for \((X, Y)\) if and only if there exists a subset \(\mathbf{Z} \subseteq \mathrm{MB}^G_{\sigma}(Y)\setminus\{X\}\) that is an adjustment set for \((X,Y)\).
\end{lemma}

\section{Algorithm}

We now introduce two rules, $R1$ and $R2$, that characterize the presence or absence of a direct causal effect from $X$ to $Y$ using only conditional independence relations restricted to the Markov blanket of $Y$, given the assumptions in Section \ref{sec:preliminaries}. 
The rules are related to the causal estimation criteria proposed in \cite{entner2013data, li2024local} for acyclic models. Our contribution extends these works to possibly cyclic SCMs.

\begin{definition}[R1]\label{def:R1}
Let $G=(\mathbf{V},\mathbf{E})$ satisfy the assumptions of Section \ref{sec:preliminaries}. Suppose there exists a variable $W \in \mathrm{MB}^G_{\sigma}(X)\setminus\{Y\}$ and a set $\mathbf{Z} \subseteq \mathrm{MB}^G_{\sigma}(Y)\setminus\{X,W\}$ such that
\[
W \not\indep Y \mid \mathbf{Z}
\quad\text{and}\quad
W \indep Y \mid \mathbf{Z} \cup \{X\}.
\]
Then it returns $\mathbf{Z}$ as a valid backdoor adjustment set for $(X,Y)$, and that $X \to Y \in \mathbf{E}$.
\end{definition}

To provide some insight regarding Definition~\ref {def:R1}, since $Y$ has no descendants, the first condition ($W \not\indep Y \mid \mathbf{Z}$) states that there is some open path from $W$ to $Y$, despite the presence of $\mathbf{Z}$. At the same time, this path can only go through $X$ to $Y$, since adding $X$ to the conditioning set closes the path by the second condition ($W \indep Y \mid \mathbf{Z} \cup \{X\}$). Therefore, $\mathbf{Z}$ closes the backdoor of $X$, and there has to be a causal effect from $X$ to $Y$. Restricting \(\mathbf{Z}\) and $W$ to be inside the Markov blankets is mainly for reducing the search space. 

\begin{definition}[R2]\label{def:R2}
Let $G=(\mathbf{V},\mathbf{E})$ satisfy the assumptions of Section \ref{sec:preliminaries}. If either of the following holds:
\begin{enumerate}
    \item $ \exists \ \mathbf{Z} \subseteq \mathrm{MB}^G_{\sigma}(Y)\setminus\{X\}$ such that $X \indep Y \mid \mathbf{Z}$,
    \item $\exists \ W \in \mathrm{MB}^G_{\sigma}(X)\setminus\{Y\}$ and $\mathbf{Z} \subseteq \mathrm{MB}^G_{\sigma}(Y)\setminus\{X,W\}$ such that $W \not\indep X \mid \mathbf{Z}
    \quad\text{and}\quad
    W \indep Y \mid \mathbf{Z}$,
\end{enumerate}
then it returns that $X \to Y \notin \mathbf{E}$.
\end{definition}

For Definition~\ref{def:R2}, $(1)$ is trivial in the sense that if there exists a set that separates $X$ and $Y$, then by $\sigma$-faithfulness, there is no causal effect. In $(2)$, the first condition ($W \not\indep X \mid \mathbf{Z}$) ensures that there is an active path from $W$ to $X$ despite the presence of $\mathbf{Z}$. The second condition ($ W \indep Y \mid \mathbf{Z}$) ensures there is no active path from $W$ to $Y$ given $\mathbf{Z}$. If there were an edge from $X$ to $Y$, one of these statements would be false. The Markov blankets are again used for shrinking the search space for $\mathbf{Z}$ and $W$.

Note that the formulations and explanations for both $R1$ and $R2$ are already agnostic to the nature of the underlying paths; without further specification, this reasoning seems to allow both cyclic and acyclic causal structures. To prove this, we need the following lemma.

\begin{lemma}[Invariance of witness variables under $\sigma$-acyclification]
\label{lem:witness-invariance}
Let $G$ be a (possibly cyclic) directed graph over $\mathbf{V}$ satisfying $\sigma$-faithfulness
, and let $G^{\mathrm{acy}}$ be a $\sigma$-acyclification of $G$. Let $X,Y \in \mathbf{O}$ and $W \in \mathbf{O} \setminus \{X,Y\}$. Then, there exists $W \in \mathrm{MB}^G_\sigma(X)\setminus\{Y\}$ satisfying the R1/R2 conditions in $G$ if and only if there exists a $W\in \mathrm{MB}^{G^{\mathrm{acy}}}_\sigma(X)\setminus\{Y\}$ satisfying the R1/R2 conditions in $G^{\mathrm{acy}}$.
\end{lemma}

With the help of Lemma~\ref{lem:witness-invariance}, our main result is that the rules remain sound and complete in possibly cyclic SCMs.

\begin{theorem}[Soundness and completeness in possibly cyclic SCMs]
\label{theom:soundcomp}
    Under our assumptions in Section \ref{sec:preliminaries},
    Let $G$ be a simple SCM with treatment-outcome pair $(X, Y)$. Then R1 and R2 are sound. Moreover, if neither R1 nor R2 applies, the causal effect of $X$ on $Y$ is not identifiable from the conditional independence relations.
\end{theorem}

\begin{remark}  
 If neither $R_1$ nor $R_2$ applies, then the available CI relations are compatible with models that disagree on the interventional distribution of $Y$ under $do(X=x)$. In one model, $X$ may have a causal effect on $Y$, in this case the estimand may be an adjustment functional of the form $P(Y\mid do(X=x))=\sum_z P(Y\mid X=x,Z=z)P(Z=z)$. In another compatible model, $X$ may have no causal effect on $Y$, i.e. $P(Y\mid do(X=x))=P(Y)$. Therefore, from CI information alone, one cannot decide which expression correctly represents the interventional distribution. This is the sense in which the rules are complete: failure of both rules means that no CI-based criterion can determine the causal effect.
 \end{remark}

Unlike in the acyclic setting, where every SCC is a singleton, cyclic graphs can contain nontrivial SCCs. It is therefore natural to ask whether $R1$ and $R2$ remain applicable when the pre-treatment assumption (Assumption~$(i)$ in Section~\ref{sec:preliminaries}) is relaxed with respect to $\mathrm{SCC}(X)$ and $\mathrm{SCC}(Y)$. Specifically, suppose that $\operatorname{De}(\mathrm{SCC}(X))\setminus\mathrm{SCC}(X)\subseteq\mathrm{SCC}(Y)$ and $\operatorname{De}(\mathrm{SCC}(Y))\setminus\mathrm{SCC}(Y)=\emptyset$. We refer to this relaxed condition as Assumption~$(i)'$.
However, since nodes within an SCC are mutually reachable, if $|\mathrm{SCC}(X)|\geq 2$ or $|\mathrm{SCC}(Y)|\geq 2$ in a simple SCM $G$, then $X$ or $Y$ has descendants within its own SCC. Consequently, there may exist acyclifications of $G$ for which the acyclic pre-treatment assumption is violated, and by extension no longer satisfy the conditions for the acyclic setting in which the rules were proposed. It is therefore not immediately clear whether the combined rules $R1$ and $R2$ remain sound and complete under this relaxation. In fact, the following lemma shows that $R1$ is no longer sound for simple cyclic SCMs under the specified condition.

\begin{lemma}\label{lem:relaxed_pre-treatment}
    Let $G$ be a simple SCM under Assumption~$(ii)$ in Section~\ref{sec:preliminaries} and the relaxed Assumption~$(i)'$. If $Y\notin\mathrm{SCC}(X)$ and $|\mathrm{SCC(X)}|\geq2$, then $R1$ is not sound.
\end{lemma}

\subsection{Local search of adjacent nodes in cyclic models}

The local component of this method is the identification of Markov blankets using conditional independence tests. Existing Markov blanket discovery algorithms, are usually defined under the assumption of acyclic graphs. At first glance, this raises the question of whether these algorithms remain valid in cyclic settings. Our result addresses this concern by showing that Markov blanket identification depends only on the underlying conditional independence structure, which is invariant under $\sigma$-acyclification.

\begin{theorem}[Validity of DAG Markov blanket algorithms in cyclic SCMs]
\label{thm:mb-algorithms-cyclic}
Let $G$ be a (possibly cyclic) directed graph over $\mathbf{V}$. 
Assuming $\sigma$-faithfulness, any sound
Markov blanket discovery algorithm for DAGs 
that relies on observed conditional independence tests returns the observed Markov blanket $\mathrm{MB}^G_\sigma(X)$ when applied to a distribution induced by $G$.
\end{theorem}

\begin{remark}
    The theorem does not assert that the underlying causal graph is acyclic. Rather, it shows that Markov blanket discovery depends only on the conditional independence assertions, which is preserved under $\sigma$-acyclification. Consequently, DAG-based algorithms recover the correct Markov blanket even when the true data-generating process contains cycles.
\end{remark}

\section{Experimental results}\label{sec:experimental_results}

In this section, we empirically evaluate the accuracy, robustness, and efficiency of the local adjustment method on simple (possibly cyclic) SCMs. Our source code is available via this \href{https://github.com/ggvanseev/NeurIPS2026_Data-Driven_Cycle-Agnostic_Causal_Effect_Estimation.git}{link}.

\paragraph{Data generation} We generate synthetic datasets from randomly sampled SCMs with latent confounding and directed cycles, and parameterize using both linear and nonlinear SCMs. We consider: linear SCMs: $\mathbf{V}:= W\mathbf{V} + \mathbf{U}$, and nonlinear SCMs: $\mathbf{V} := tanh(W\mathbf{V}) + \mathbf{U}$, where $W$ is the weighted adjacency matrix and $\mathbf{U}$ is the vector of exogenous noises. All models are restricted to simple SCMs, ensuring well-defined solutions despite the presence of cycles. 
Additional details can be found in Appendix~\ref{ap:datagen}. 

We vary the structural properties of the SCMs. We generate six graph sizes with $8$, $15$, $25$, $50$, $100$, and $250$ nodes, and $12$, $19$, $40$, $78$, $150$, and $360$ edges, respectively. Moreover, the number of samples is $1k$, $5k$, $10k$, and $15k$ per graph. Latent confounding is introduced by randomly hiding a subset of variables (respectively, by the number of nodes, we make the following number of variables latent per graph size: $2$, $3$, $4$, $10$, $10$, and $30$), always excluding the treatment and outcome.  Exogenous noises are sampled independently with standard deviations drawn from $\mathcal{U}(0.5,1.0)$, under three settings: Gaussian noise, uniform noise (i.e., non-Gaussian noise), or mixed (node-wise random assignment to either Gaussian or uniform noise).

To probe different regimes, 
we generate datasets by combining four factors: graph structure (acyclic vs.\ cyclic), functional form (linear vs.\ nonlinear), edge condition (presence or absence of $X \to Y$), and noise type (Gaussian, non-Gaussian, or mixed). 

For each configuration, we sample 25 independent graphs per noise type, resulting in 75 graphs per (graph structure, functional form, edge condition) setting. Overall, this yields 600 graphs per graph size. This total decomposes as follows: 300 graphs per functional form; within each functional form, 150 graphs per graph structure; and within each graph structure, 75 graphs per edge condition. We repeat this procedure for each graph size.

\paragraph{Markov blankets}  The Markov blankets are estimated using the TC \cite{pellet2008using} 
for graphs with up to 50 nodes due to the limited scalability, and Fast-IAMB \cite{yaramakala2005speculative} for larger graphs, which we select based on its favorable trade-off between CI tests and computational time. 
We also study the impact of using other Markov blanket discovery algorithms, such as IAMB \cite{tsamardinos2003algorithms} and HITON-MB \cite{aliferis2003hiton} (see Appendix \ref{ap:mb} for a detailed comparison).

\paragraph{Evaluation metrics} 

Let $\mathbf{Z}$ be the set of all valid adjustment sets (ground truth), and $\mathbf{\hat Z}$ the set of all adjustment sets returned by the method. Let $e \in \{0,1\}$ indicate the presence of an edge $X \to Y$ in the true graph, and $\hat e$ the output of the method. Let $n$ be the number of graphs where $R1$ or $R2$ applies. We evaluate performance using:
Relative Error (RE), defined as
$RE=\frac{|\hat{CE} - CE|}{|CE|}$; Precision, defined as
$Pre=\frac{|\mathbf{\hat Z} \cap \mathbf{Z}|}{|\mathbf{\hat Z}|}$; and Edge Fraction accuracy (EF), defined as
$EF= \frac{1}{n}\sum_{i=1}^{n}\mathbf{1}[e_i = \hat e_i]$.



Relative error quantifies the accuracy of the estimated causal effect ($ \hat{CE}$) and is reported only when $R1$ applies in linear SCMs (see more details in Appendix~\ref{ap:ce}). 
Precision, i.e., the proportion of valid adjustment sets returned, is also computed under $R1$. 
The edge fraction accuracy evaluates performance across both $R1$ and $R2$ by measuring agreement on the presence of a causal edge. We emphasize that precision and edge fraction accuracy are not standard metrics in this context. However, we include them to enable evaluations in nonlinear settings, where the RE does not apply. The non-linear causal effect depends on the specific treatment value and, therefore, cannot be quantified by a single scalar. Finally, we report the fraction of empty instances (neither $R1$ nor $R2$ applies), which reflects cases where the method is unable to return a definitive output (see Appendix~\ref{ap:emptyfrac}).

\paragraph{Results}

Since existing local adjustment methods are derived under acyclicity assumptions, they do not provide guarantees in cyclic SCMs. We therefore study 
how the method behaves when extended to cyclic settings by comparing to the acyclic case as a baseline.

All the graphs discussed in this section are constructed such that they satisfy the pre-treatment assumption (see Appendix \ref{ap:nonpre} for the setting where this assumption is violated). 



Figure \ref{fig:RE_smallgraphs} reports RE for linear SCMs across cyclic and acyclic settings. Increasing the sample size improves estimation accuracy in both settings. More interestingly, RE is lower in cyclic graphs compared to acyclic ones, often with smaller variance. This is notable because the original local adjustment method was developed for acyclic models. Moreover, the results empirically support our theoretical claim that the method extends naturally to cyclic SCMs.

\begin{figure}[h]
    \centering
    \begin{subfigure}{0.24\textwidth}
        \centering
        \includegraphics[width=\linewidth]{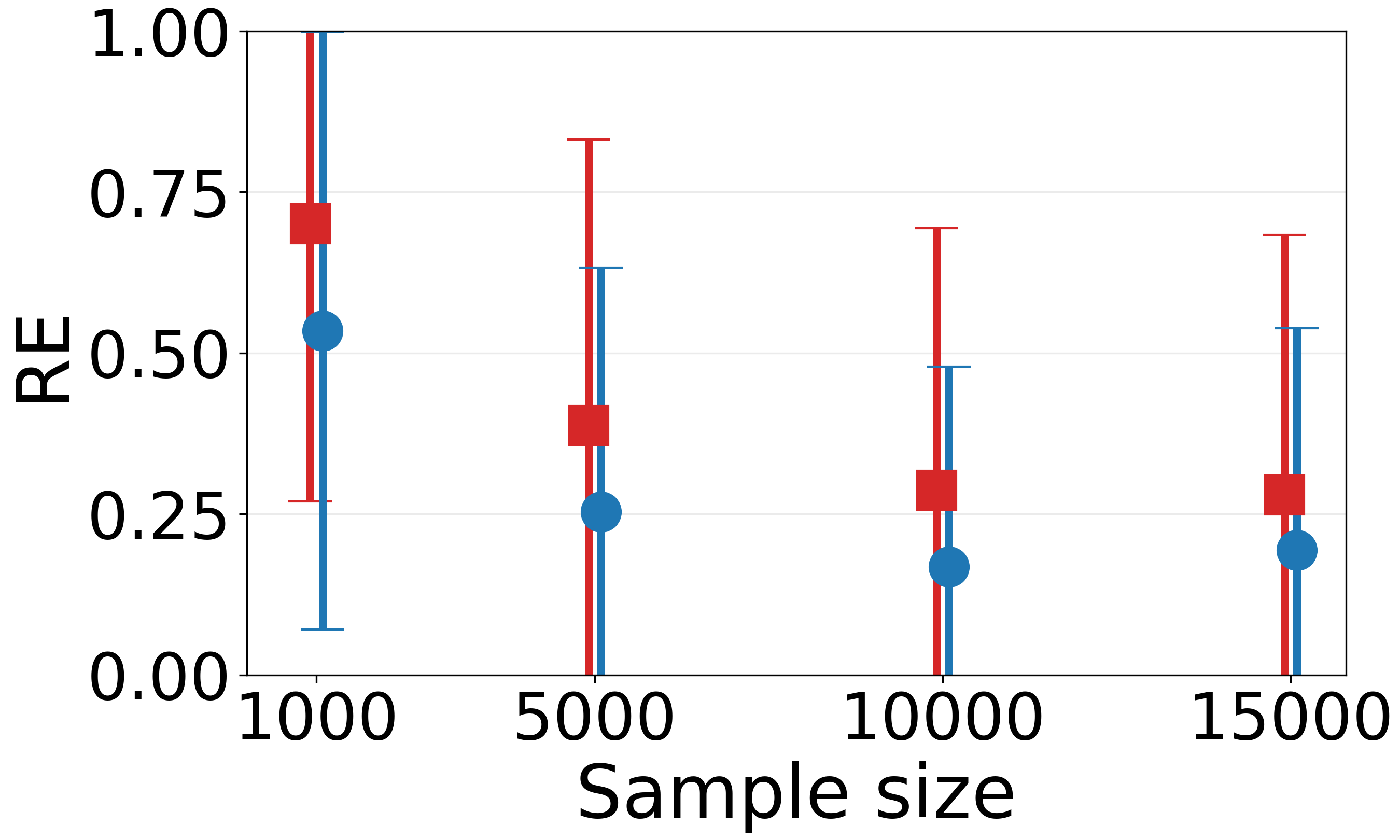}
        \caption{8 nodes}
    \end{subfigure}
    \begin{subfigure}{0.24\textwidth}
        \centering
        \includegraphics[width=\linewidth]{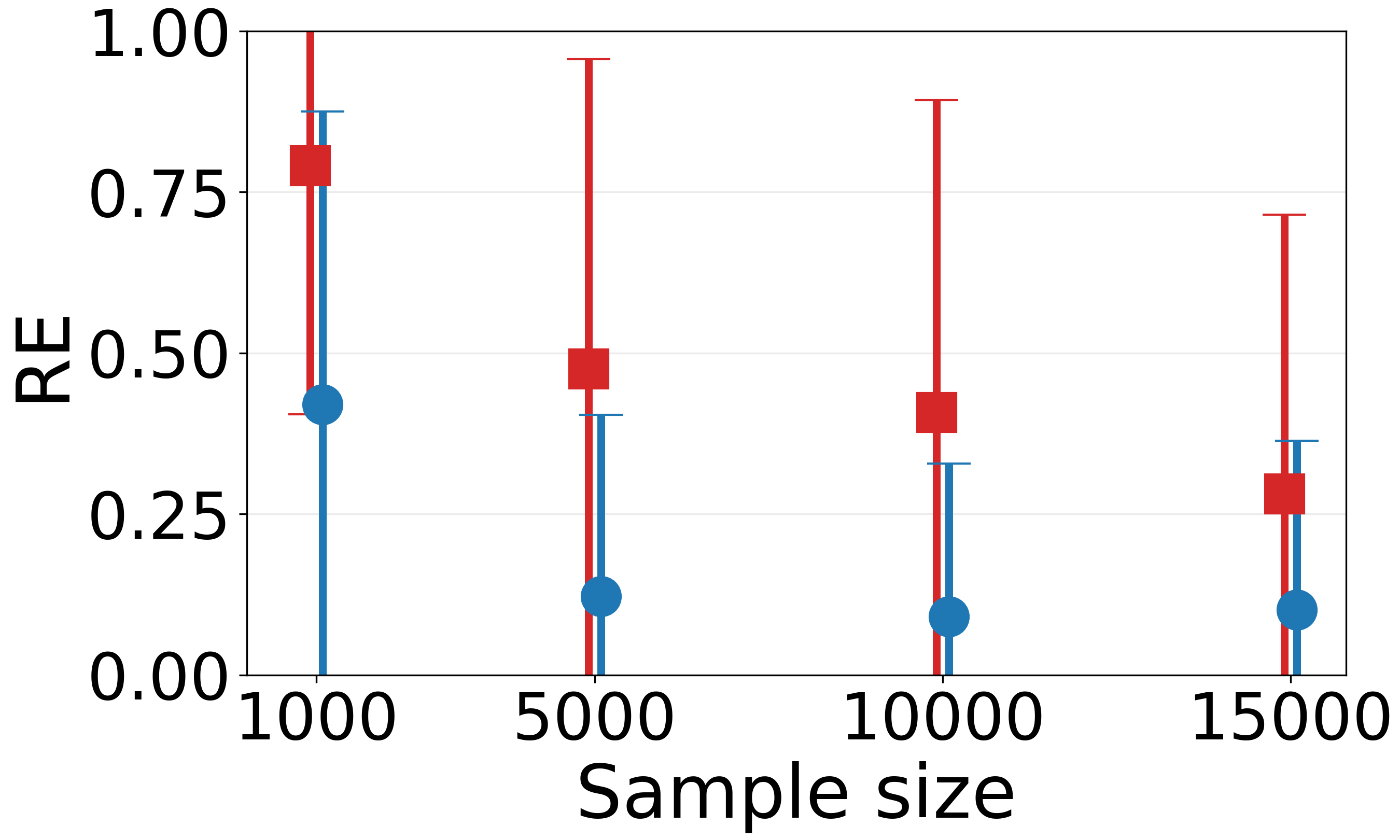}
        \caption{15 nodes}
    \end{subfigure}
    \begin{subfigure}{0.24\textwidth}
        \centering
        \includegraphics[width=\linewidth]{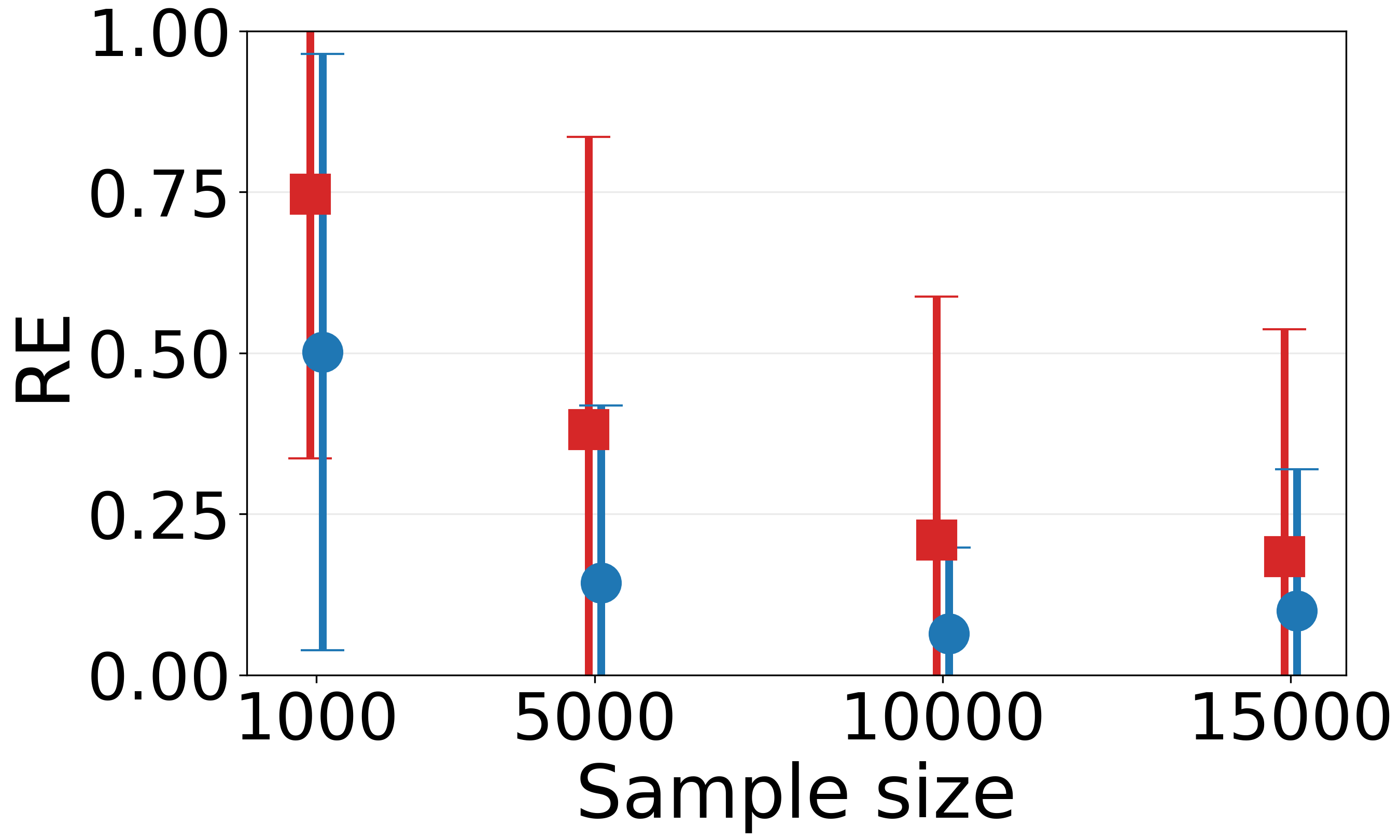}
        \caption{25 nodes}
    \end{subfigure}
    \begin{subfigure}{0.24\textwidth}
        \centering
        \includegraphics[width=\linewidth]{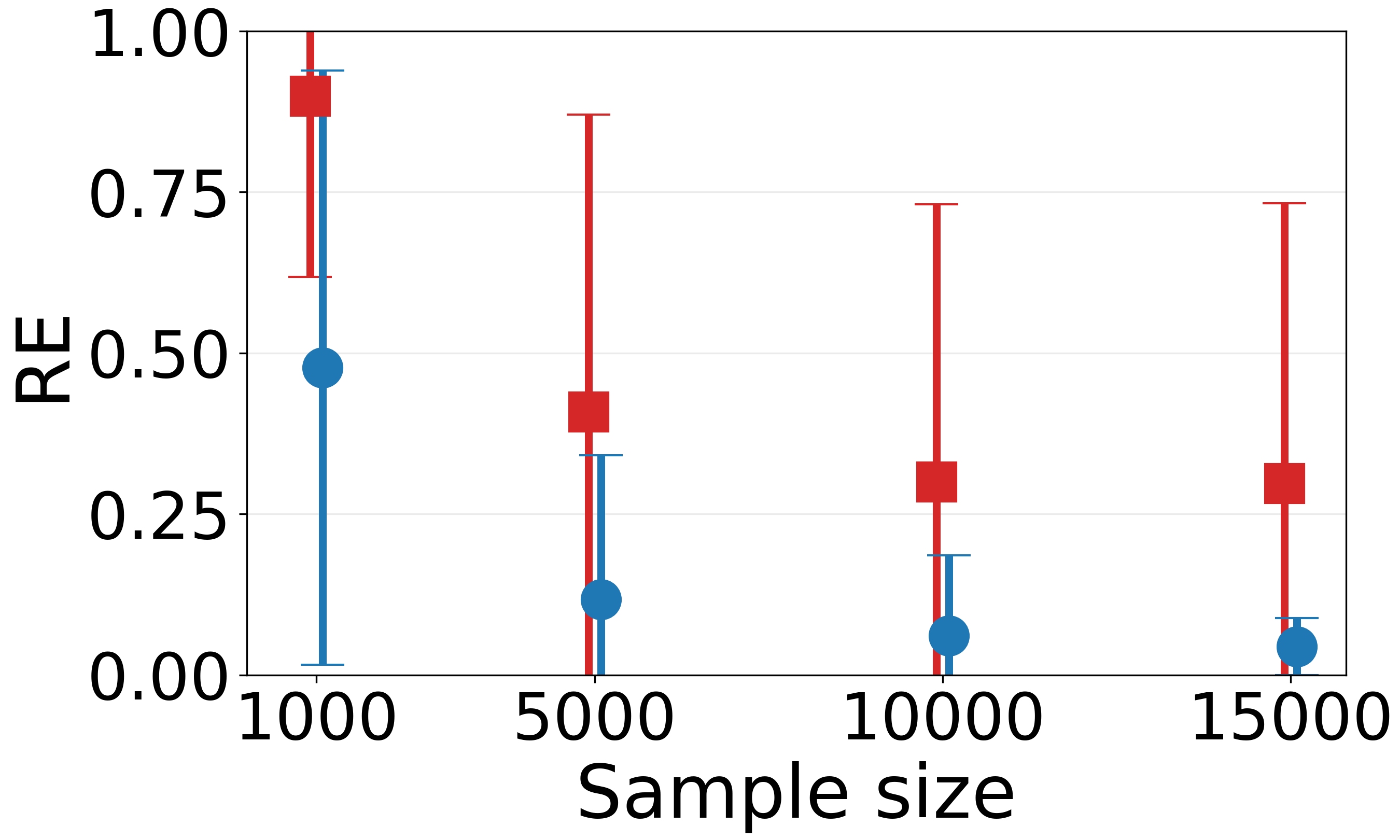}
        \caption{50 nodes}
    \end{subfigure}
    
    \caption{Relative error of linear models in acyclic (red square) and cyclic (blue circle) settings across sample sizes. Markers show the mean, with error bars indicating one standard deviation. Subfigures correspond to graph sets with different numbers of nodes.}
    \label{fig:RE_smallgraphs}
\end{figure}

\begin{figure}[h]
    \centering
    \begin{subfigure}{0.24\textwidth}
        \centering
        \includegraphics[width=\linewidth]{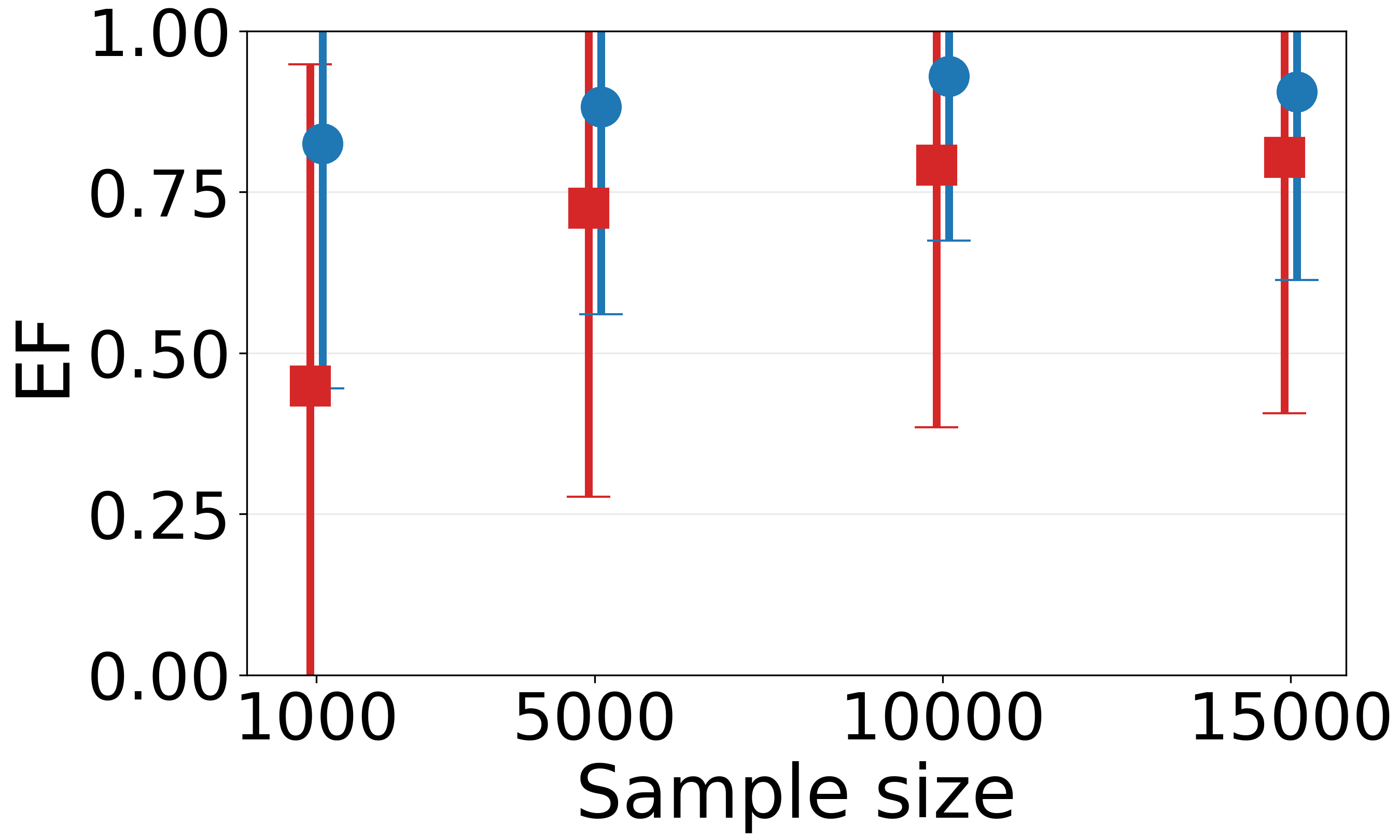}
        \caption{8 nodes}
    \end{subfigure}
    \begin{subfigure}{0.24\textwidth}
        \centering
        \includegraphics[width=\linewidth]{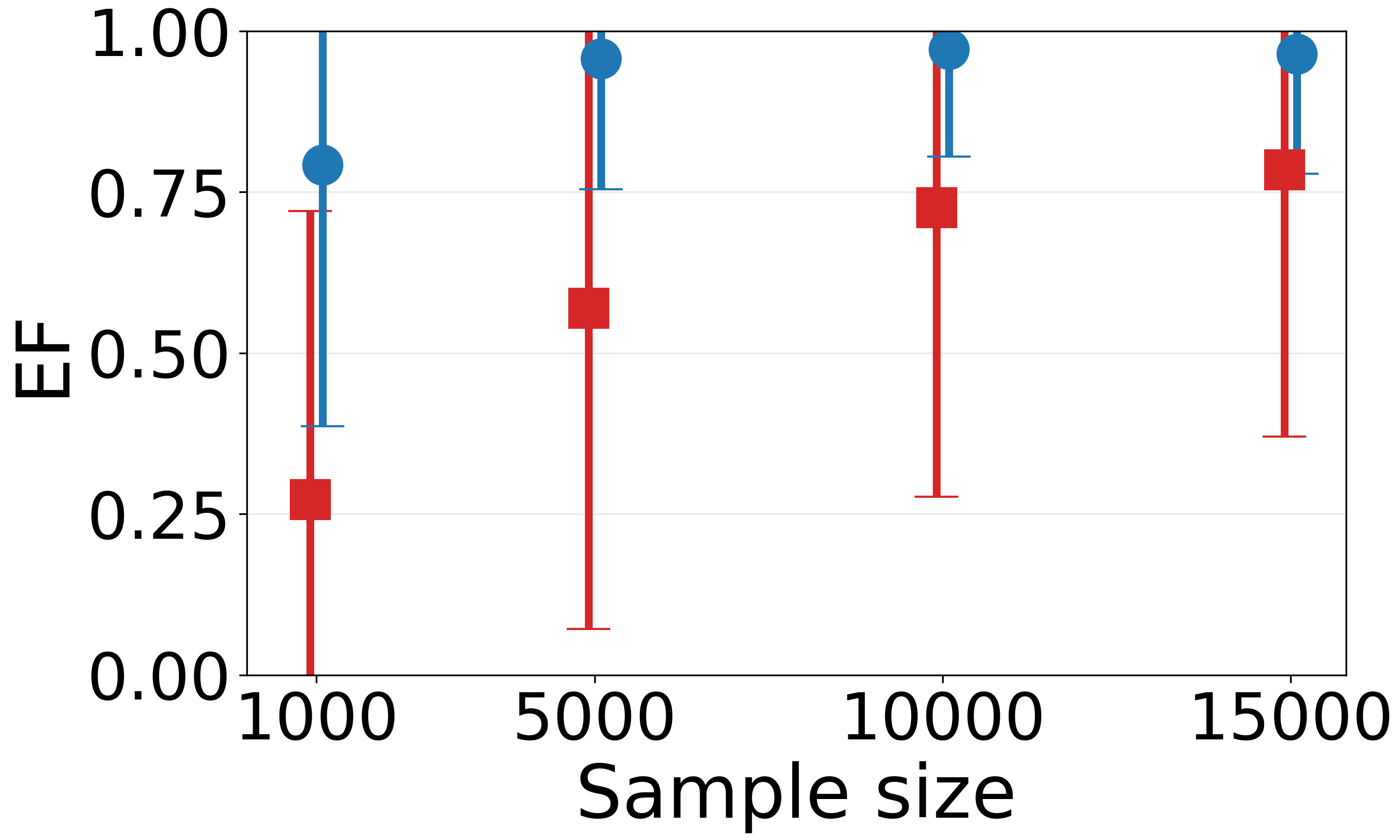}
        \caption{15 nodes}
    \end{subfigure}
    \begin{subfigure}{0.24\textwidth}
        \centering
        \includegraphics[width=\linewidth]{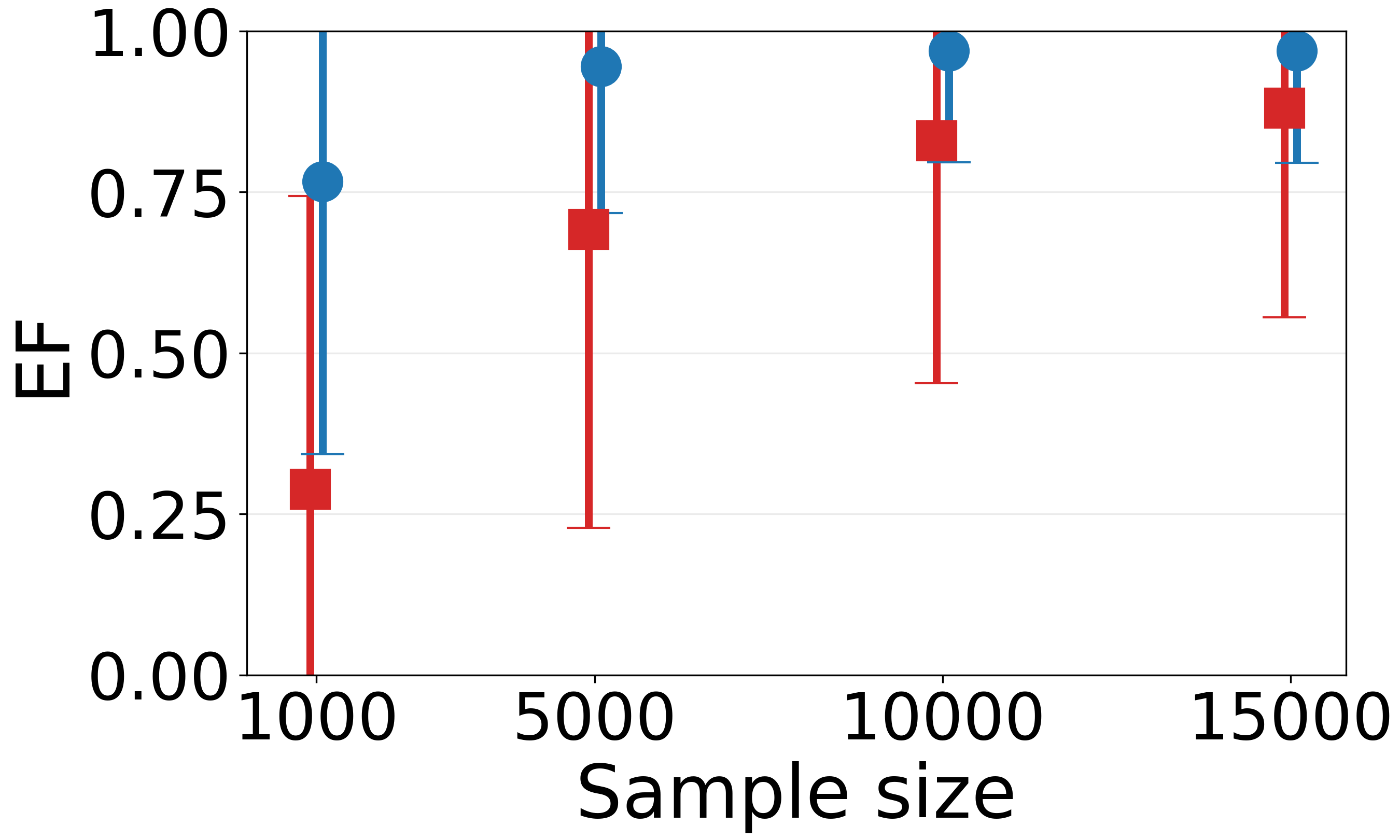}
        \caption{25 nodes}
    \end{subfigure}
    \begin{subfigure}{0.24\textwidth}
        \centering
        \includegraphics[width=\linewidth]{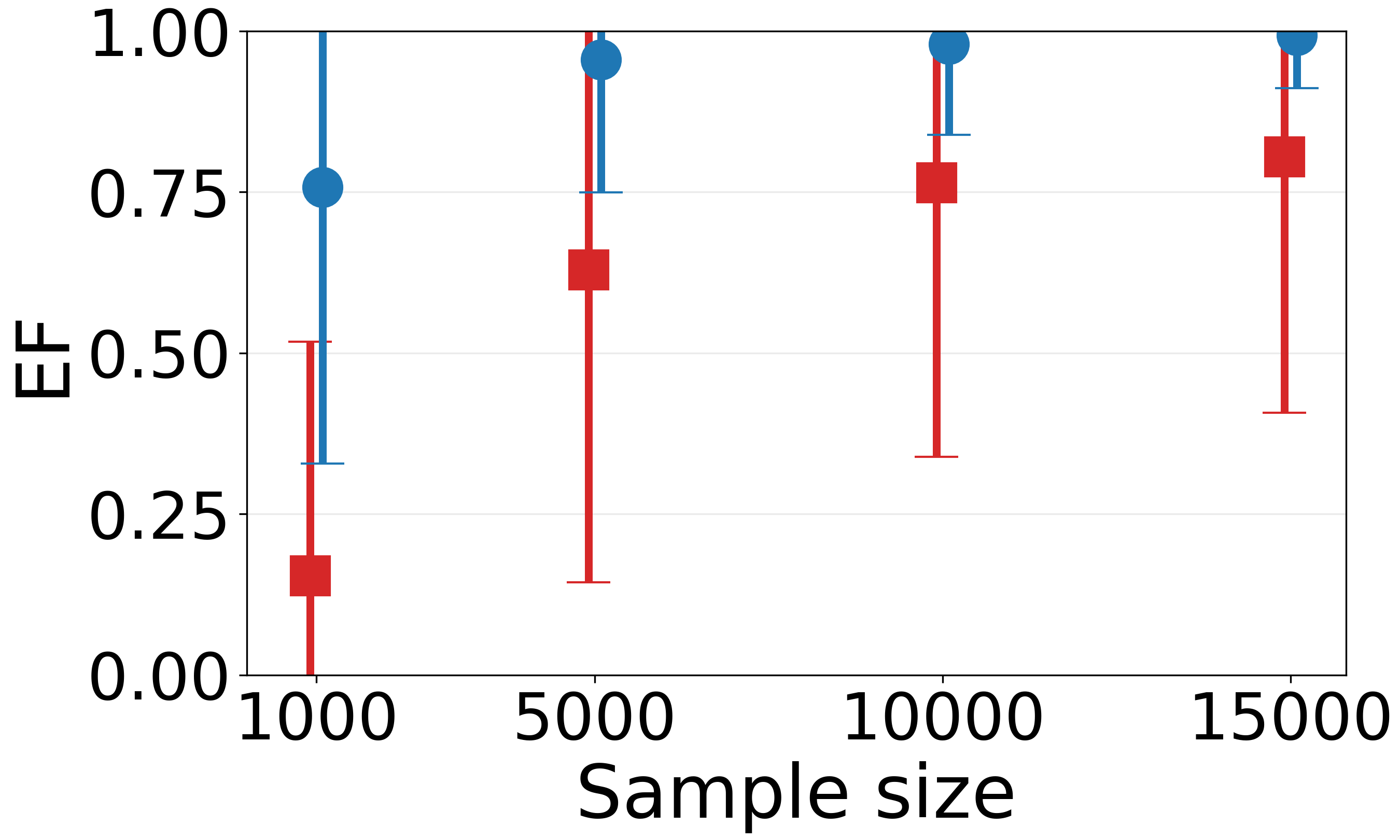}
        \caption{50 nodes}
    \end{subfigure}

        \par\medskip

    \begin{subfigure}{0.24\textwidth}
        \centering
        \includegraphics[width=\linewidth]{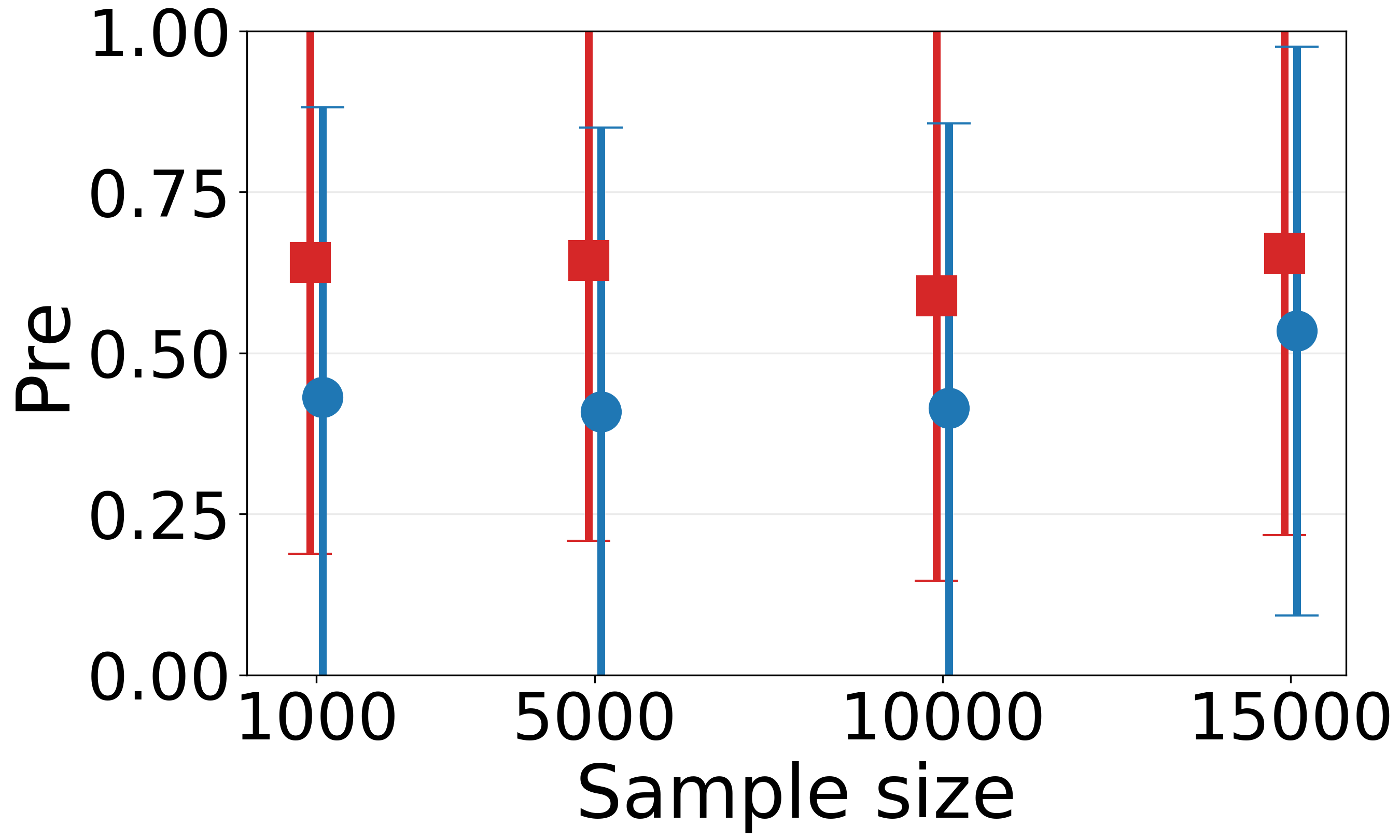}
        \caption{8 nodes}
    \end{subfigure}
    \begin{subfigure}{0.24\textwidth}
        \centering
        \includegraphics[width=\linewidth]{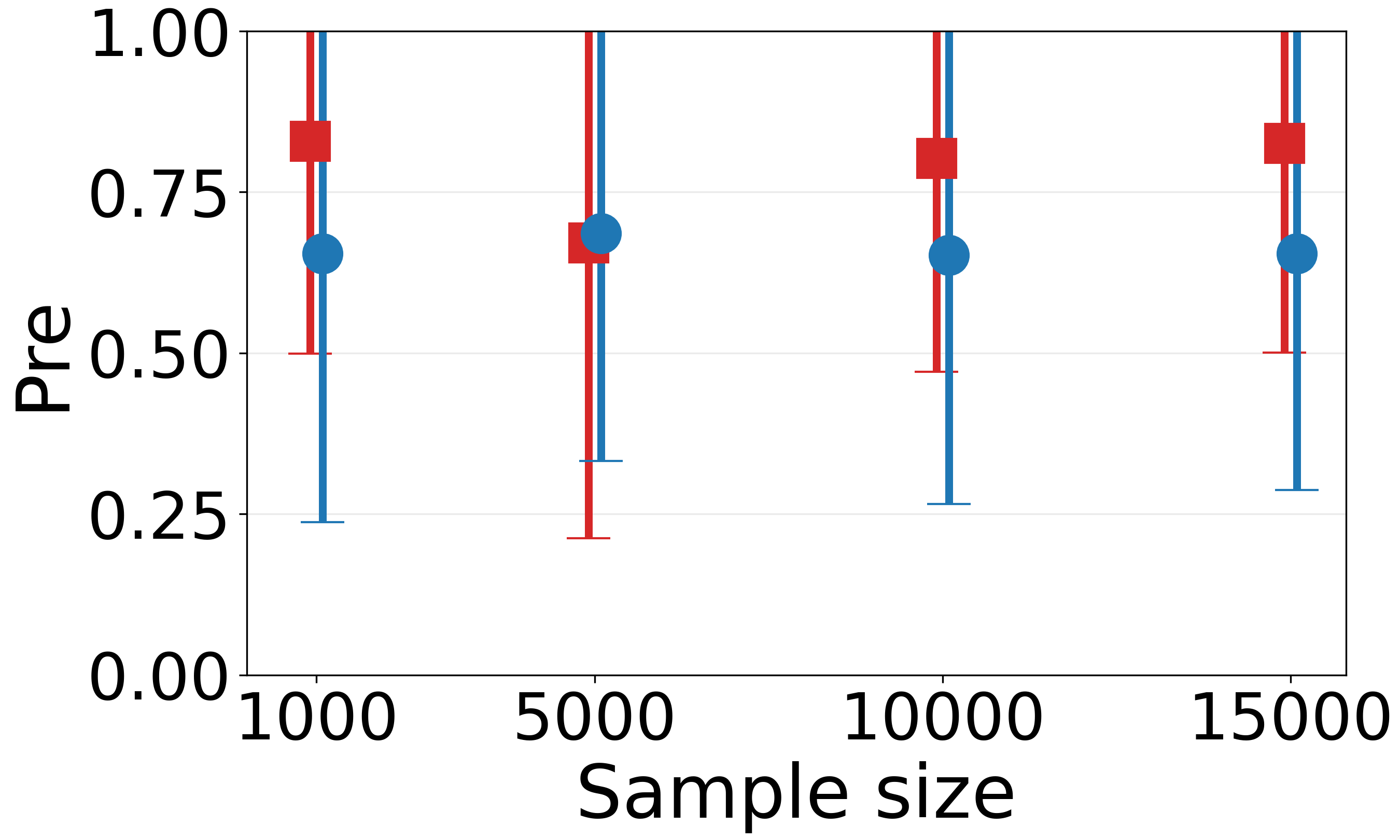}
        \caption{15 nodes}
    \end{subfigure}
    \begin{subfigure}{0.24\textwidth}
        \centering
        \includegraphics[width=\linewidth]{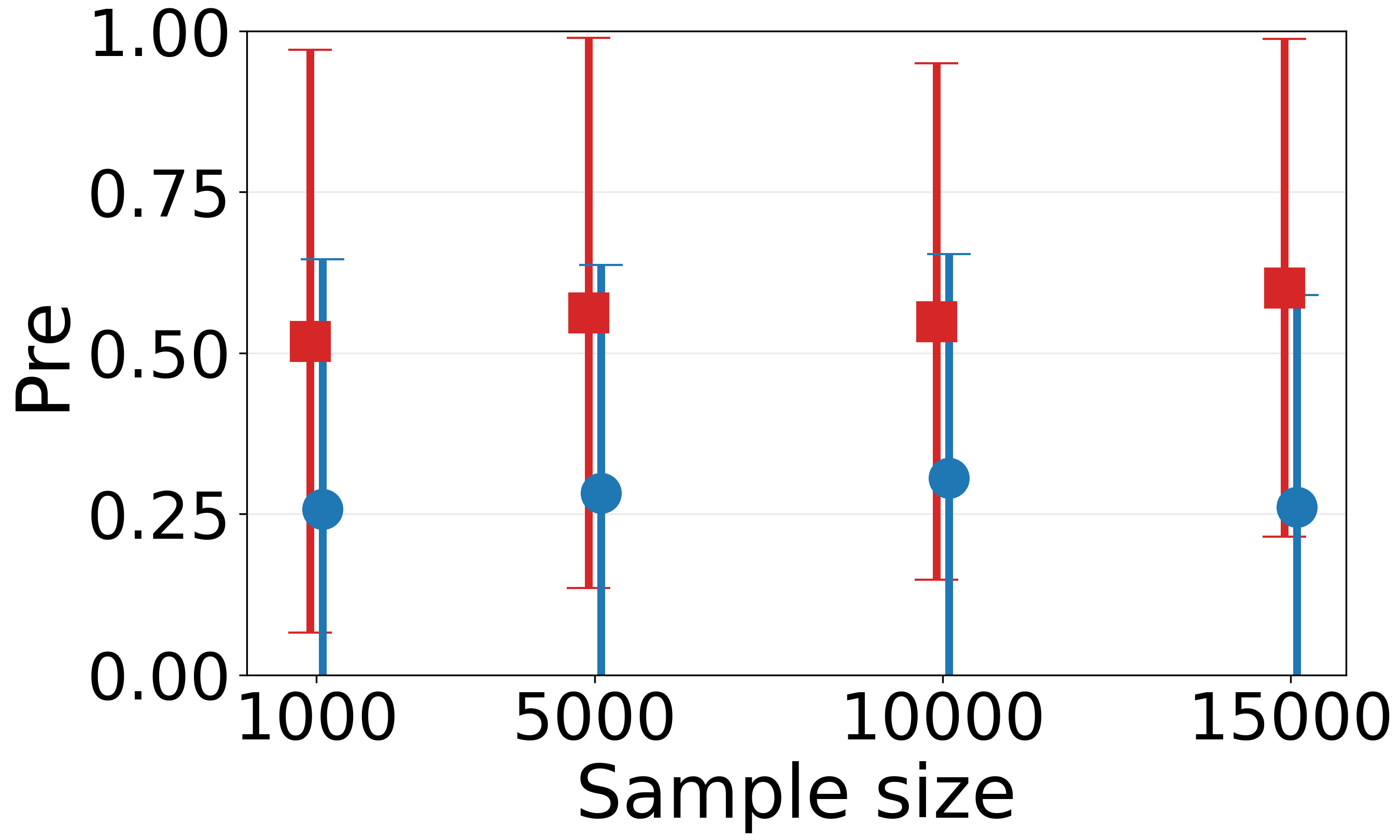}
        \caption{25 nodes}
    \end{subfigure}
    \begin{subfigure}{0.24\textwidth}
        \centering
        \includegraphics[width=\linewidth]{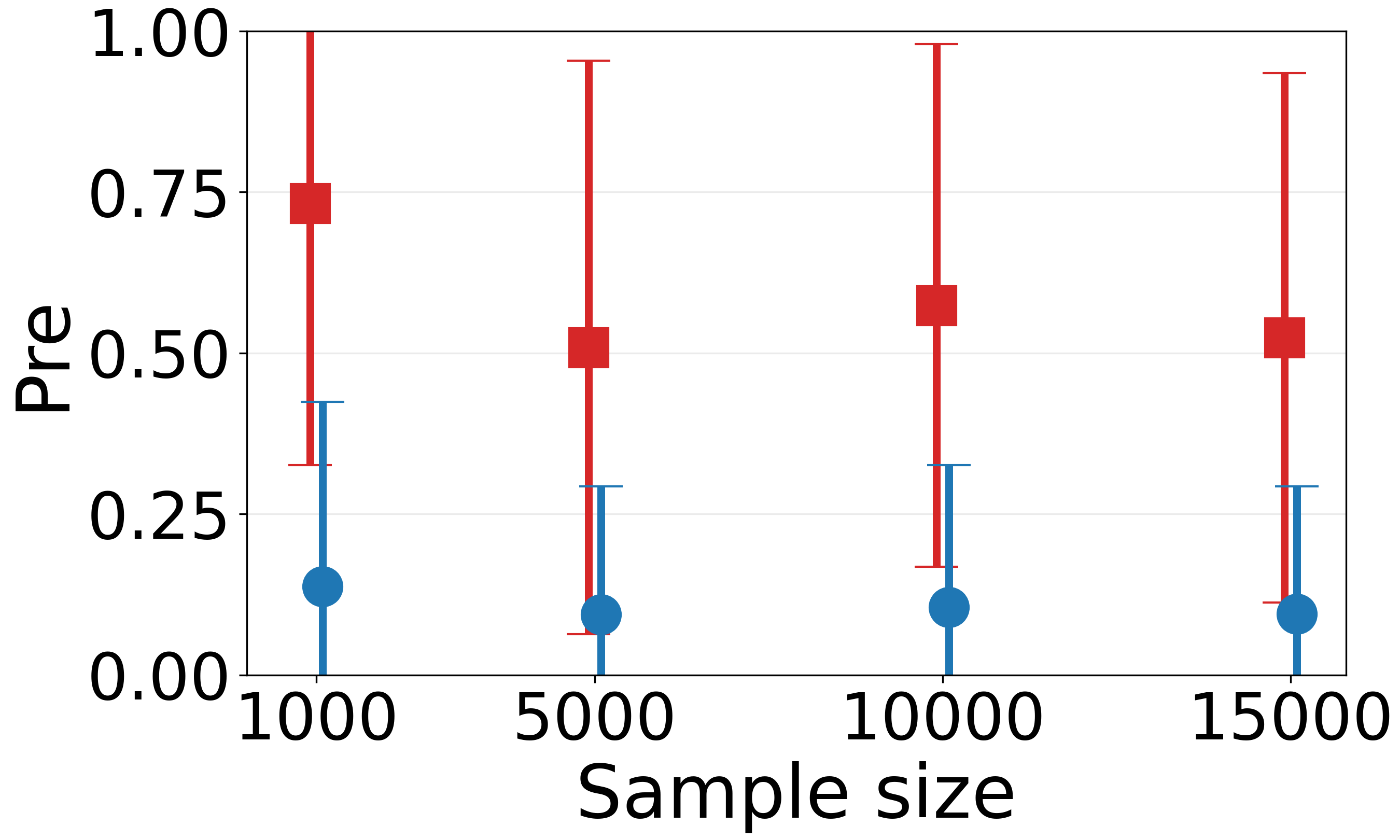}
        \caption{50 nodes}
    \end{subfigure}
    \caption{Edge fraction accuracy (first row) and precision (second row) of linear and non-linear models in acyclic (red square) and cyclic (blue circles) settings across sample sizes. Markers show the mean, with error bars indicating one standard deviation. Subfigures correspond to graph sets with different numbers of nodes.}
    \label{fig:nonlinear_smallgraphs}
\end{figure}

Figure \ref{fig:nonlinear_smallgraphs} reports the edge fraction accuracy across both linear and nonlinear SCMs. Similar to the RE results, performance improves monotonically with increasing sample size. Cyclic graphs consistently achieve higher EF compared to acyclic graphs, often approaching near-perfect recovery in larger sample regimes.

The precision results reveal a different behavior. While cyclic settings achieve lower RE and EF, acyclic graphs in general obtain higher precision. Notable, this reduction in precision does not translate into worse causal estimation or edge detection; cyclic configurations still achieve better RE and EF overall. Furthermore, the sample size seems to have little effect on the precision.



Finally, the empty fraction, the proportion of graphs for which the algorithm does not apply either rule, remains in general below $20\%$ for the cyclic setting, while in the acyclic case, it usually (with an exception for the 15 nodes configuration) stays around $25\%$ (see more details in Appendix \ref{ap:emptyfrac}).

Additional experiments on larger graphs reported in Appendix \ref{ap:largegraphs} show the same qualitative trends, suggesting that the observed behavior is stable across graph scales.

\paragraph{Discussion}

Our experiments reveal two important phenomena. First, the proposed local adjustment method remains effective in cyclic SCMs in terms of the relative error, the standard metric in the literature. This empirically supports the theoretical results established in this work. Second, cyclic settings often exhibit lower relative error and improved edge recovery compared to acyclic settings, while simultaneously having lower precision. The mechanism underlying this trade-off is unclear. Since our goal is to show that existing methods can provide cycle-agnostic causal effect estimation, we leave a theoretical explanation of this phenomenon to future work. However, we do note that precision is a very sensitive metric, where the addition or absence of a single variable relative to a true valid adjustment set is already penalized as heavily as a set without any overlap. Hence, while this metric is easy to interpret, it seems too sensitive to align well with causal effect estimation performance.

Regarding the consistency of our implementation’s performance in acyclic settings relative to prior work, we make two observations. To begin, the experiments exhibit relatively large standard deviations. While this variability may initially appear concerning, it is consistent with prior observations reported by \cite{entner2013data}, who similarly reported substantial variance. This suggests that the instability is not unique to our implementation, but may instead reflect inherent variance in the problem. To continue, for acyclic graphs, the average relative error aligns with the results reported by \cite{li2024local}, supporting the validity of our implementation and evaluation protocol. In contrast, for cyclic graphs, direct comparisons with other works and methods are not possible due to the absence of prior work.

\section{Conclusion}

We addressed the problem of covariate selection for nonparametric causal effect estimation in the presence of cycles and latent variables. While existing approaches rely on acyclicity assumptions, many real-world scenarios exhibit feedback loops that violate the requirement. In this work, we show that a local, data-driven covariate selection method that allows causal effect estimation under a pre-treatment assumption \cite{li2024local} remains valid in cyclic settings.

Our main theoretical contribution establishes that the soundness and completeness guarantees of this method continue to hold in simple structural causal models. Notably, this result holds in a nonparametric setting, requiring no assumptions about the presence or absence of cycles. As a consequence, it is cycle-agnostic and applies uniformly to both acyclic and cyclic systems.

Empirically, we validated our approach on a wide range of synthetic data, including linear and nonlinear models with varying graph sizes, noise levels, and cyclic and acyclic cases. The results demonstrate consistent performance across settings, with notably strong accuracy and stability in cyclic regimes. 

An important direction for future work is to better understand the empirical and theoretical differences between cyclic and acyclic settings, particularly the apparent trade-off between estimation accuracy and edge recovery versus adjustment-set precision. In addition, we could explore finding the optimal set among the recovered valid adjustment sets. Finally, extending local adjustment methods beyond the pre-treatment assumption while preserving soundness and completeness also remains an important open problem in the cyclic setting.

\small

\bibliographystyle{unsrt}
\bibliography{references}


\appendix

\section{Markov blanket discovery algorithms} \label{ap:mb}

We considered four standard Markov blankets discovery algorithms: IAMB \cite{tsamardinos2003algorithms}, Fast-IAMB \cite{yaramakala2005speculative}, HITON-MB \cite{aliferis2003hiton} and TC \cite{pellet2008using}. All methods aim to recover the Markov blanket $\mathrm{MB}(X)$ of a target variable $X$ using conditional independence (CI) tests.

The TC algorithm relies on an exhaustive characterization: a variable $Y$ belongs to $\mathrm{MB}(X)$ iff $X$ and $Y$ are dependent conditional on all remaining variables. This characterization is sound and complete under faithfulness, however it requires CI tests with conditioning sets of size $|\mathbf{V}|−2$, which can be computationally expensive and statistically unreliable under large graphs regimes.

In contrast, the other methods adopt a grow-shrink strategy that incrementally constructs the Markov blanket. The IAMB algorithm alternates between the forward phase, where variables associated with $X$ are added, and a backward phase, where it removes the false positives via CI tests. HITON-MB  refines the strategy by enforcing stricter inclusion criteria, performing more tests before including a new candidate. Lastly, Fast-IAMB prioritizes variables with the strongest conditional association in the forward phase and only considers conditional independence tests if there is a significant number of samples. Since we are using Fisher's Z test, we require $N>|\boldsymbol{Z}|+3$ where $\boldsymbol{Z}$ is the valid adjustment set \cite{JMLR:v8:kalisch07a}. We note that for our sample sizes starting at 1000 samples, while we have at most 250 nodes in total, this condition is always satisfied.

\begin{figure}[h]
    \centering
    \includegraphics[width=0.7\linewidth]{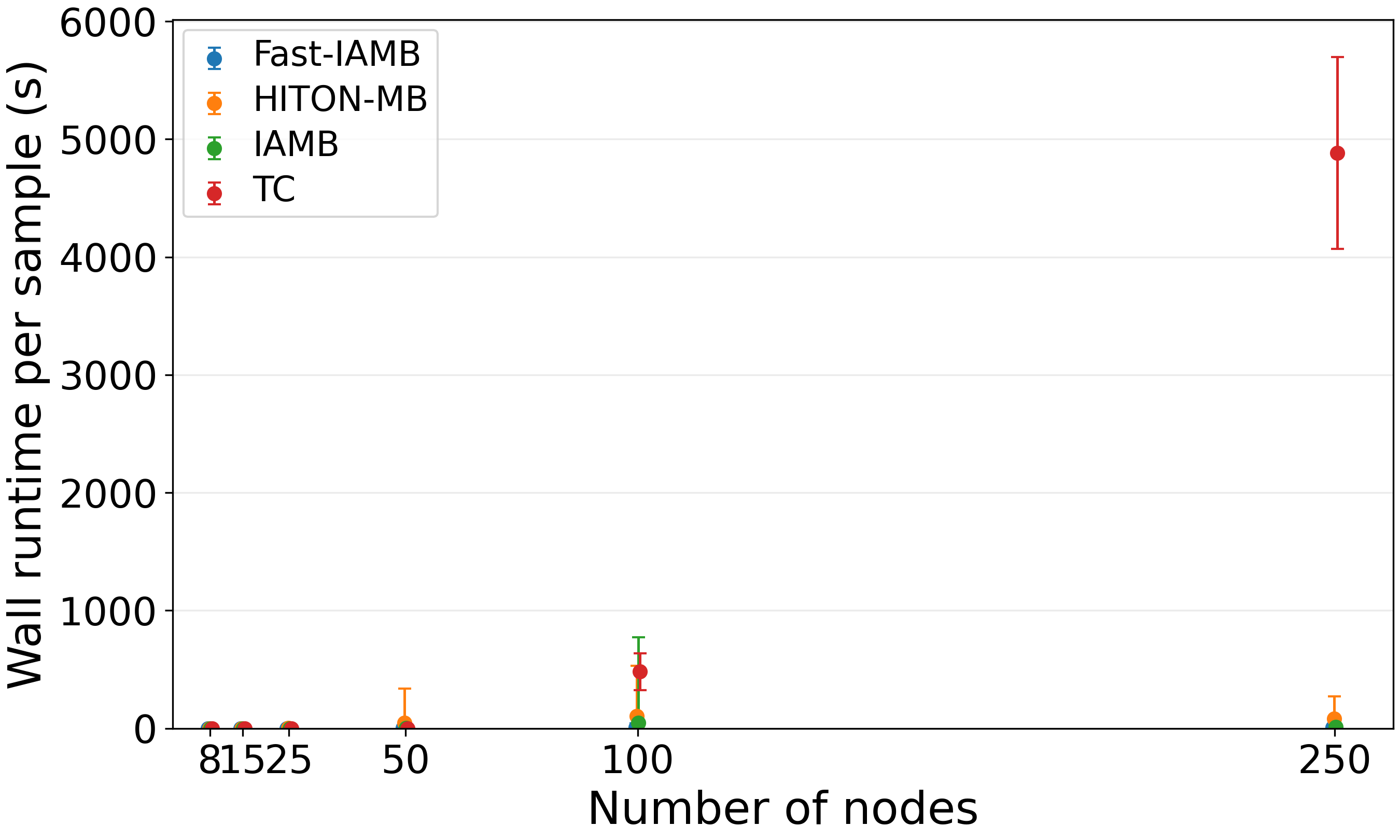}
    \caption{The y-axis shows the wall run time in seconds, while the x-axis is the number of nodes. The figure shows the mean runtime with one standard deviation for Fast-IAMB (blue), HITON-MB (orange), IAMB (green) and TC (red). For the smaller number of nodes, up to 50, the algorithms run in similar times, with HITON-MB sometimes taking the longest. However, as the number of nodes increases, TC starts to become very slow compared to the other algorithms.}
    \label{fig:app_runtime_by_nodes_mb_comparison}
\end{figure}

Empirical results in Figure \ref{fig:app_runtime_by_nodes_mb_comparison}
show that for small graphs ($|\mathbf{V}| \leq 50$), all algorithms exhibit comparable runtimes, with HITON-MB occasionally being slower due to its conservative inclusion strategy. However, as the number of nodes increases, TC becomes significantly slower than the other methods due to the large conditioning sets involved. Fast-IAMB shows better performance in the larger settings, combining a moderate number of CI tests with consistently small conditioning sets. Based on these observations, we use TC 
for small graphs, where its exhaustive characterization remains tractable, and Fast-IAMB 
for larger graphs.


\begin{remark}
    All conditional independence statements in the algorithms are evaluated using a statistical test at significance level $\alpha$. The CI test $\mathrm{CI}(X,Y \mid \mathbf{S}; \alpha)$ returns true if $X$ and $Y$ are judged conditionally independent given $\mathbf{S}$, i.e. $X \indep Y \mid \mathbf{S}$.
\end{remark}

Furthermore, Figure~\ref{fig:comparing_mb} shows a comparison of the Markov blanket discovery algorithms on the relative error and the edge fraction accuracy. The experimental setup is the same as for the smaller graphs discussed in Section~\ref{sec:experimental_results}. The figure shows that all Markov blanket discovery algorithms have approximately the same performance per (a)cyclic configuration. Based on this, we assume that all algorithms seem to perform equally well on our relevant metrics, making it possible to pick whichever is preferred based on other considerations.

\begin{figure}[h]
    \centering
    \begin{subfigure}{0.45\textwidth}
        \centering
        \includegraphics[width=\linewidth]{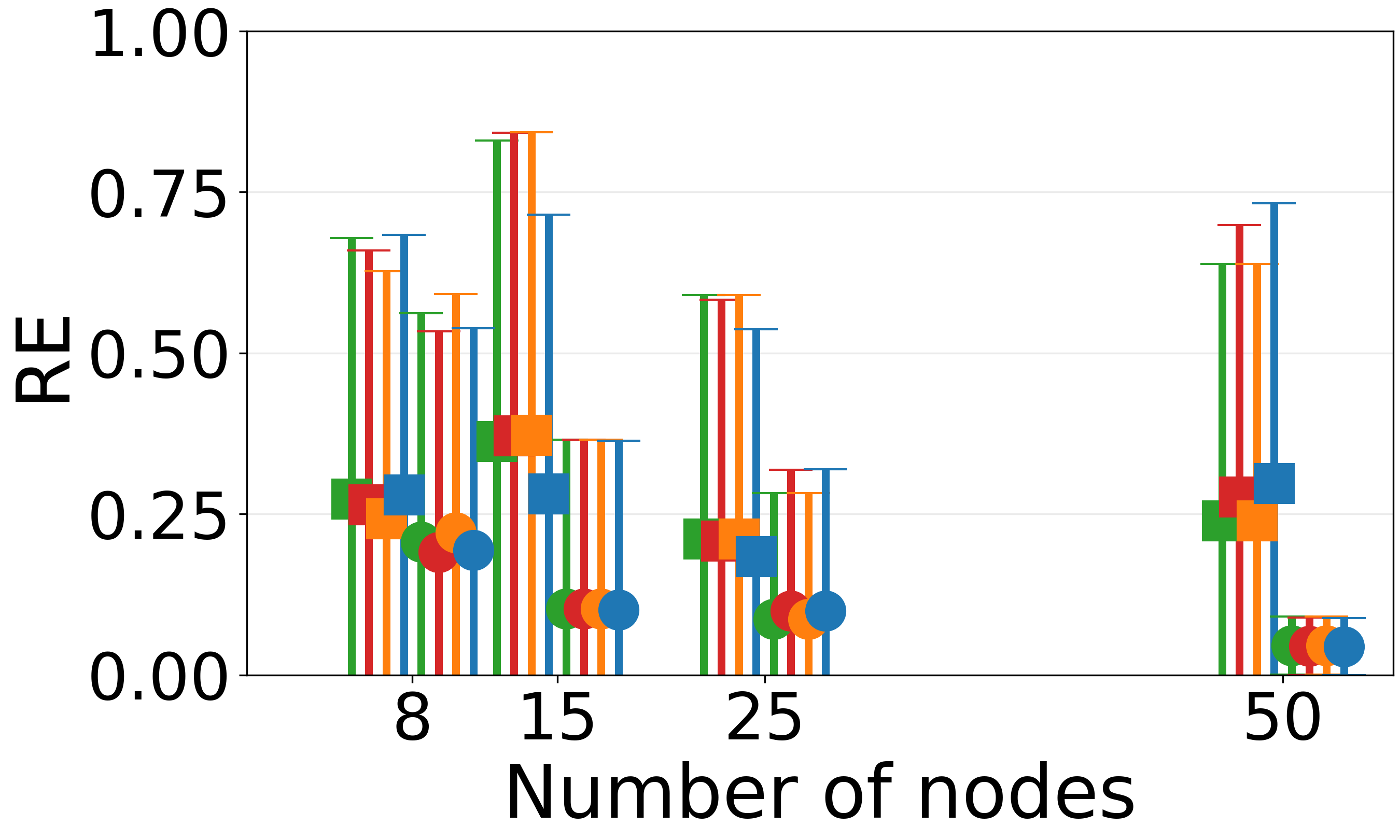}
        \caption{Relative error for linear models}
    \end{subfigure}
    \begin{subfigure}{0.45\textwidth}
        \centering
        \includegraphics[width=\linewidth]{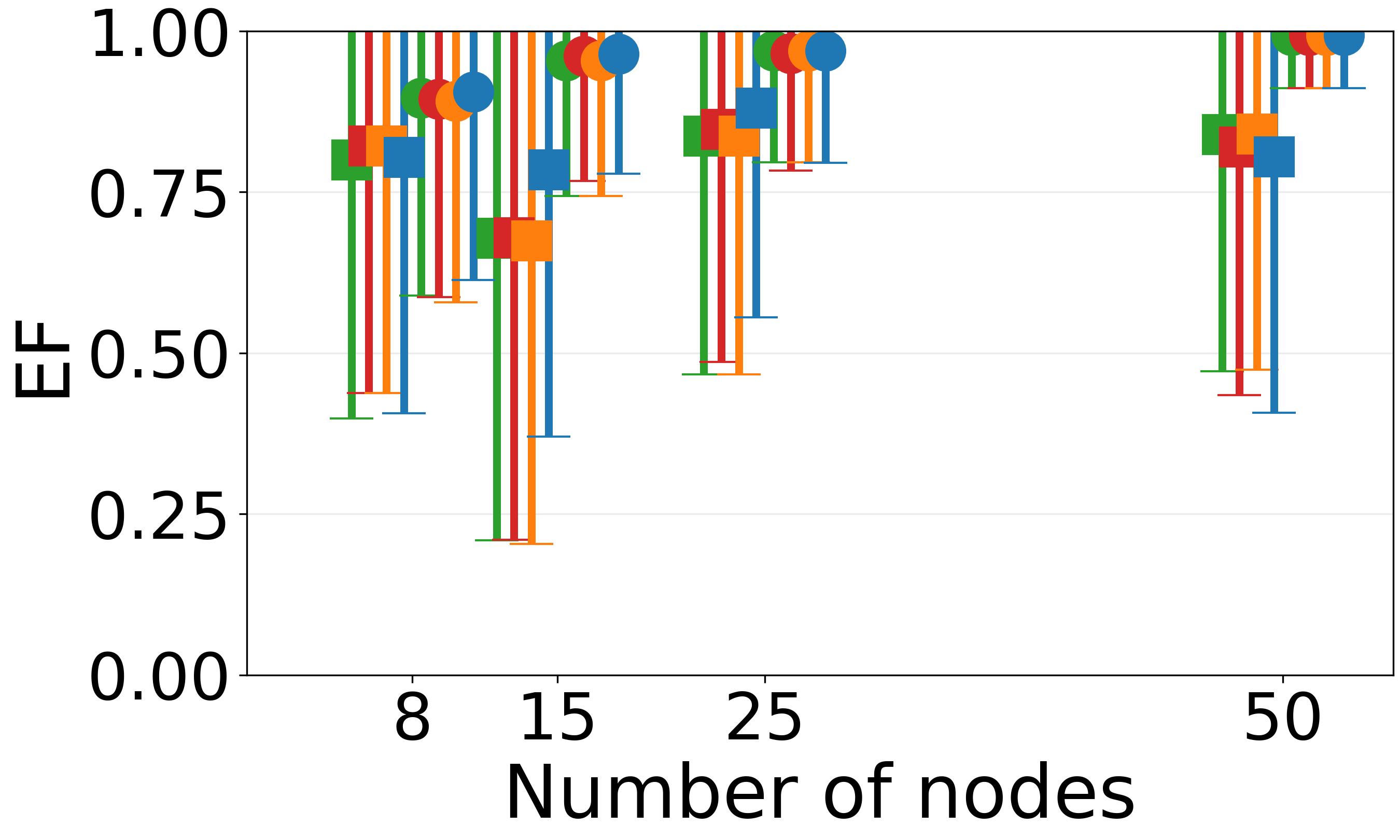}
        \caption{Edge fraction acc. for (non-)linear models}
    \end{subfigure}
    \caption{Performance across varying graph sizes of the different Markov blanket discovery algorithms. The x-axis shows the number of nodes, while the y-axis shows the evaluated metric. Results are reported for acyclic (squares) and cyclic (circles) configurations, where the offset from left to right corresponds to Fast-IAMB (green), HITON-MB (red), IAMB (orange) and TC (blue). Markers indicate mean values, and error bars denote one standard deviation. Each subfigure corresponds to a different performance metric.}
    \label{fig:comparing_mb}
\end{figure}

Finally, Figure~\ref{fig:comparing_mb_ntest} shows the number of independence tests across varying graph sizes of the different Markov blanket discovery algorithms for a sample size of $15K$. In both graph structures, HITON-MB (second in the offset) is considerably worse than Fast-IAMB, IAMB and TC. The latter are quite similar.  Comparing the cyclic and acyclic cases, the cyclic configuration has a larger number of independence tests and variance compared to the acyclic configuration across the board.

\begin{figure}[h]
    \centering
    \begin{subfigure}{0.45\textwidth}
        \centering
        \includegraphics[width=\linewidth]{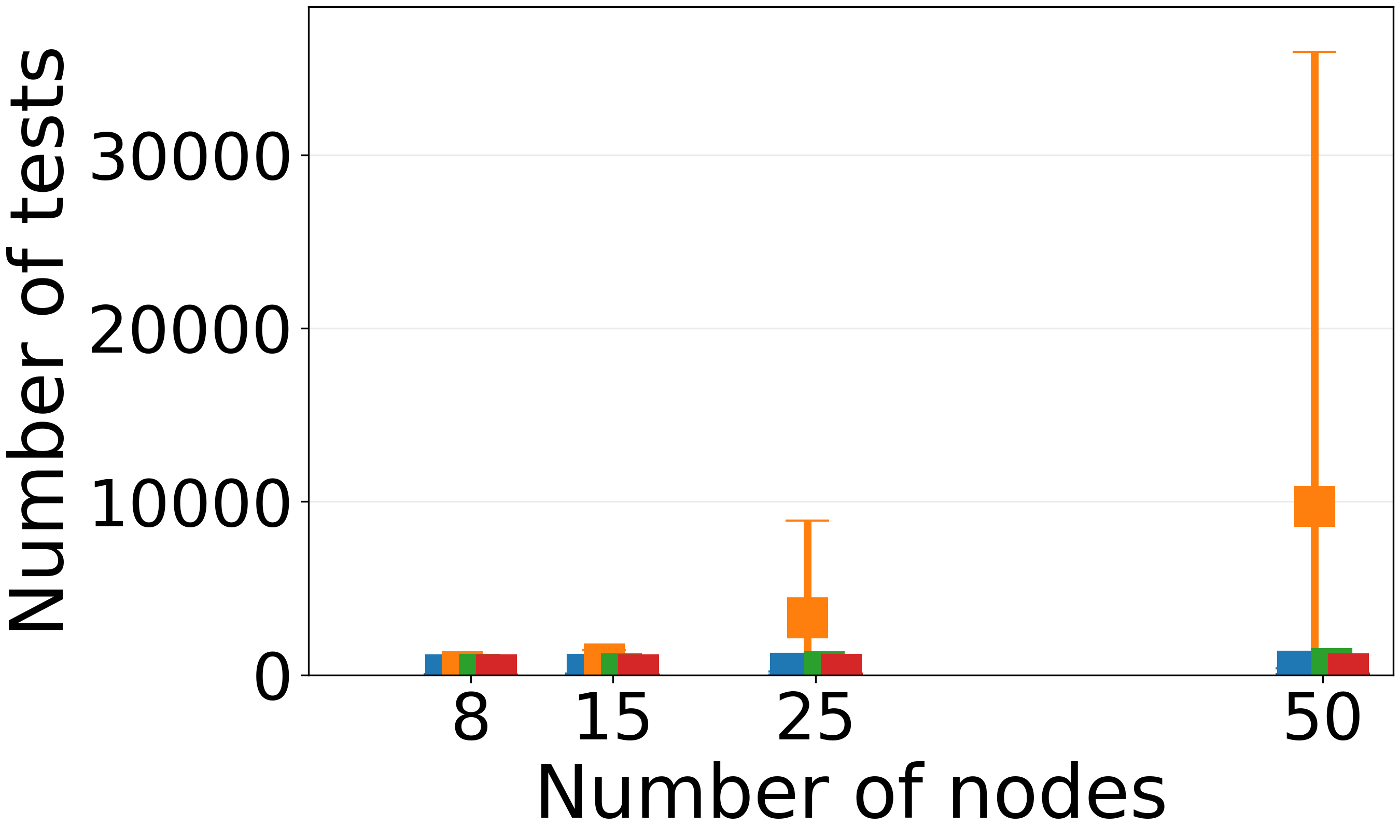}
        \caption{Acyclic}
    \end{subfigure}
    \begin{subfigure}{0.45\textwidth}
        \centering
        \includegraphics[width=\linewidth]{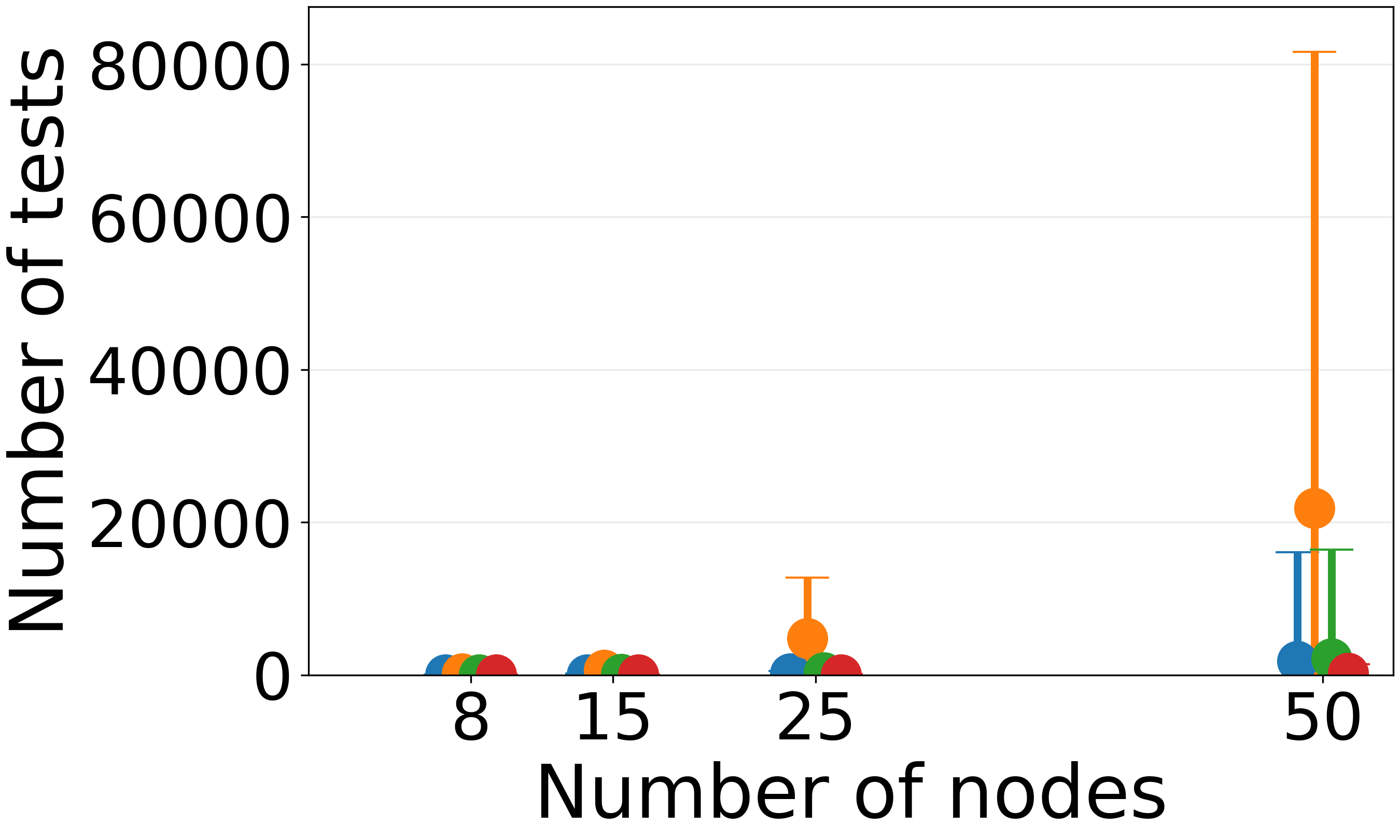}
        \caption{Cyclic}
    \end{subfigure}
    \caption{Number of independence tests across varying graph sizes of the different Markov blanket discovery algorithms for a sample size of $15K$. The x-axis shows the number of nodes, while the y-axis shows the number of independence tests. Results are reported for acyclic (squares) and cyclic (circles) configurations, where the offset from left to right corresponds to Fast-IAMB (blue), HITON-MB (orange), IAMB (green) and TC (red). Markers indicate mean values, and error bars denote one standard deviation. Each subfigure corresponds to a different graph structure.}
    \label{fig:comparing_mb_ntest}
\end{figure}

\begin{algorithm}[h]
\label{alg:fast-iamb}
\caption{Fast-IAMB Markov Blanket Discovery \cite{yaramakala2005speculative}}
\KwIn{Data $D$, variables $\mathbf{V}$, target $X$, significance level k}
\KwOut{$\widehat{MB}(X)$}

$\widehat{MB}(X) \leftarrow \emptyset$\;
$\mathbf{A} \leftarrow \{Y \in \mathbf{V} \setminus \{X\} : X \not\indep Y \mid \emptyset\}$\;

\While{$\mathbf{A} \neq \emptyset$}{
    $\mathrm{MB}_{\text{old}} \leftarrow \widehat{MB}(X)$\;
    \textit{insufficient\_data} $\leftarrow$ False\;
    Sort $\mathbf{A}$ by decreasing association with $X$ given $\widehat{MB}(X)$\;

    \ForEach{$Y \in \mathbf{A}$}{
        \If{Significant$\big(X,Y,\widehat{MB}(X),k \big)$ = $\mathrm{True}$}{
        $\widehat{MB}(X) \leftarrow \widehat{MB}(X) \cup \{Y\}$\;
        }
        \Else{
            \textit{insufficient\_data} $\leftarrow$ True;
        }  
    }

    \ForEach{$Y \in \widehat{MB}(X)$}{
        \If{$X \indep Y \mid \widehat{MB}(X) \setminus \{Y\}$}{
            $\widehat{MB}(X) \leftarrow \widehat{MB}(X) \setminus \{Y\}$\;
        }
    }

    $\mathbf{A} \leftarrow \{Y \in \mathbf{V} \setminus (\{X\} \cup \widehat{MB}(X)) : X \not\indep Y \mid \widehat{MB}(X)\}$\;

    \If{\textit{insufficient\_data} = $\mathrm{True}$ $\mathrm{\mathbf{and}}$ $\widehat{MB}(X) = MB_{\text{old}}$}{
        break\;
    }
}
\Return $\widehat{MB}(X)$\;
\end{algorithm}

\begin{algorithm}[h]
\label{alg:tc}
\caption{TC Markov Blanket Discovery \cite{pellet2008using}}
\KwIn{Data $D$, variables $\mathbf{V}$, target $X$, significance level $\alpha$}
\KwOut{$\widehat{MB}(X)$}

$\widehat{MB}(X) \leftarrow \emptyset$\;
$n \leftarrow |D|$\;
$q \leftarrow \left\lfloor n / 10 \right\rfloor$\;

\If{$q > 0$}{
    $\alpha' \leftarrow \alpha / (10q)$\;
}
\Else{
    $\alpha' \leftarrow \alpha$\;
}

\ForEach{$Y \in \mathbf{V} \setminus \{X\}$}{
    $\mathbf{S} \leftarrow \mathbf{V} \setminus \{X, Y\}$\;

    \If{$X \not\indep Y \mid \mathbf{S}$ at significance level $\alpha'$}{
        $\widehat{MB}(X) \leftarrow \widehat{MB}(X) \cup \{Y\}$\;
    }
}
\Return $\widehat{MB}(X)$\;
\end{algorithm}

\section{Local adjustment set search}

In this section, we outline the procedure of the local causal adjustment set search. The algorithm LSAS \cite{li2024local} is shown in Algorithm \ref{alg:lsas} and uses $\Theta$ to store the estimated causal effect of $X$ on $Y$. If $\Theta = \emptyset$, the procedure is unable to identify a valid adjustment set. If $\Theta = 0$, the algorithm concludes that there is no causal effect. Otherwise, $\Theta$ contains an estimate of the causal effect.

\begin{algorithm}[h]
\caption{Local Search Adjustment Sets (LSAS) \cite{li2024local}}
\label{alg:lsas}
\KwIn{Observational data $D$, treatment $X$, outcome $Y$}
\KwOut{Estimated causal effect $\Theta$}

$\widehat{MB}(X), \widehat{MB}(Y) \leftarrow$ Markov Blanket Discovery$(X, Y, D)$\;

$\Theta \leftarrow \emptyset$\;

\ForEach{$W \in \widehat{MB}(X) \setminus \{Y\}$}{
    \ForEach{$\mathbf{Z} \subseteq \widehat{MB}(Y) \setminus \{X\}$}{
        
        \If{$(W,Z)$ satisfies R1 \ref{def:R1}}{
            Estimate $\theta = \mathbb{E}[Y \mid X, \mathbf{Z}]$\;
            $\Theta \leftarrow \theta$\;
        }
        
        \If{$(W,\mathbf{Z})$ satisfies R2 \ref{def:R2}}{
            \Return $\Theta \leftarrow 0$\;
        }
    }
}

\Return $\Theta$\;
\end{algorithm}

\section{Proofs}\label{ap:proofs}

We first introduce maximal ancestral graphs (MAGs), which generalize DAGs to settings with latent variables.

\begin{definition}
    A maximal ancestral graph (MAG) $\mathcal{M}_{\mathbf{O}} = (\mathbf{O}, \mathbf{E})$ over observed variables $\mathbf{O}$ is a mixed graph containing directed edges ($\rightarrow$), representing direct causal relations between observed variables, and bidirected edges ($\leftrightarrow$), encoding the presence of latent common causes. A MAG satisfies that for every pair of non-adjacent nodes, there exists a conditioning set that $m$-separates them. 

    Conditional independence in MAGs is characterized through $m$-separation, which generalizes $d$-separation in DAGs.
\end{definition}

\begin{lemma}[Invariance of Markov blanket between DAGs and MAGs]
\label{lem:eqmbmag}
    Let $G$ be a DAG possibly containing latent variable, and let
    $\mathcal M_{\mathbf O}$ be the MAG over $\mathbf O$ obtained by latent
    projection of $G$. Therefore:
    \[
    \mathrm{MB}^\mathcal{M_{\mathbf O}}_{m}(Y) = \mathrm{MB}^{G}_{d}(Y)
    \]
\end{lemma}

\begin{proof}
    By \cite{zhang2008causal}, the observed conditional independence assertions induced by $G^{\mathrm{acy}}$ is exactly
    the $m$-separation model of $\mathcal M_{\mathbf O}$.
    By the latent projection \cite{zhang2008causal}, for all disjoint
    $A,B,S \subseteq \mathbf O$,
    \[
    A \indep_d B \mid \mathbf{S} \text{ in } G^{\mathrm{acy}}
    \quad \Longleftrightarrow \quad
    A \indep_m B \mid \mathbf{S} \text{ in } \mathcal M_{\mathbf O}.
    \]
    Therefore, a set $\mathbf S \subseteq \mathbf O \setminus \{X\}$ separates
    $X$ from all observed variables outside $\mathbf S\cup\{X\}$ in
    $G^{\mathrm{acy}}$ if and only if it does so in $\mathcal{M}_{\mathbf O}$.
    Since the collection of separating sets is identical in both graphs, their
    minimal sets coincide.
\end{proof}

\subsection{Lemma \ref{lem:acy-transfer-sep}}
\begin{proof}
The independence models induced by $\sigma$-separation in $G$ and d-separation in $G^{\mathrm{acy}}$ coincide, i.e., $\operatorname{IM}_\sigma(G)=\operatorname{IM}_d(G^{\mathrm{acy}})$ \cite{mooij2020constraint}. The result follows immediately from the definition of these models.



\end{proof}

\subsection{Lemma \ref{lem:mb-acy}}
\begin{proof}
Let $\mathbf{S} := \mathrm{MB}^G_{\sigma}(Y)$. By definition,
$Y \indep_{\sigma} \mathbf{O} \setminus (\mathbf{S} \cup \{Y\}) \mid \mathbf{S}$ in $G$. By Lemma~\ref{lem:acy-transfer-sep}, this is equivalent to $Y \indep_{d} \mathbf{O} \setminus (\mathbf{S} \cup \{Y\}) \mid \mathbf{S}$ in $G^{\mathrm{acy}}$, so $\mathbf{S}$ is a separating set in $G^{\mathrm{acy}}$, implying $\mathbf{S} \supseteq MB^{G^{\mathrm{acy}}}_d(Y)$ by minimality. The reverse inclusion follows symmetrically, yielding equality.
\end{proof}

\subsection{Lemma \ref{lem:edge-preserved}}
\begin{proof}
    Since $\mathrm{De}(Y) = \emptyset$, node $Y$ has no outgoing edges in $G$. Therefore, no directed cycles can involve $Y$. Hence $\mathrm{SCC}(Y)=\{Y\}$ and no acyclification step introduces edges out of $Y$. This implies $\mathrm{Ch}(Y)=\emptyset$ in any acyclification.

    Now consider $X$. The assumption $\mathrm{De}(X)\subseteq\{Y\}$ implies that $X$ cannot lie on a directed cycle (as $\mathrm{Ch}(Y) = \emptyset$). Hence, $\mathrm{SCC}_{G}(X)=\{X\}$.
    It follows that no acyclification step modifies edges incident to $X$ inside a SCC. Therefore, the only possible child of $X$ in $\hat G^{\mathrm{acy}}$ is $Y$, and we have $ \mathrm{Ch}_{\hat G^{\mathrm{acy}}}(X)\subseteq\{Y\}$.
    Finally, consider the edge $X\to Y$. Since both $x$ and $y$ are singleton SSCs, the $\sigma$-acyclification rule for edges between distinct components preserves the edge $X\to Y$ if and only if it is present in $G$.
\end{proof}

\subsection{Lemma \ref{lem:equiv-backdoor}}
\begin{proof}
(``if'' part) Assume that $\mathbf{Z}$ satisfies Definition~\ref{def:backdoor-adjustment}. Hence: (i) $\mathbf{Z} \cap \mathrm{De}(X)=\emptyset$, and (ii) $\mathbf{Z}$ $\sigma$-blocks every backdoor path from $X$ to $Y$.

We first show condition (1) from Lemma \ref{lem:equiv-backdoor}: since the intervention node $I_X$ has no incoming edges and a single outgoing edge  $I_X \to X$, every path from $I_X$ to an observed variable must first pass through $X$, hence, any $\sigma$-open path from $I_X$ to a node $Z_i \in \mathbf{Z}$ would imply a $\sigma$-open path from $X$ to $Z_i$, which in particular requires that $\mathbf{Z}$ lies in the descendant region of $X$. This contradicts (i). Therefore, $\mathbf{Z} \indep_\sigma I_X.$

Condition (2) from Lemma \ref{lem:equiv-backdoor}: Any path from $I_X$ to $Y$ must begin with $I_X \to X$. After reaching $X$, either the path leaves $X$ via an outgoing edge, in which case it is blocked by conditioning on $X$ (Note $\mathrm{SCC}(X)=\{X\}$, since for $X$ the only possible descendant $Y$ has no descendants), or it reaches $X$ via an incoming edge when viewed as a path from $X$ to $Y$, corresponding to a backdoor path. All such backdoor paths are $\sigma$-blocked by $\mathbf{Z}$ by (ii). Hence, all paths from $I_X$ to $Y$ are $\sigma$-blocked given $X$,$\mathbf{Z}$, and thus $Y \indep_\sigma I_X \mid X,\mathbf{Z}$.


(``only if'' part) Assume that $\mathbf{Z} \indep_\sigma I_X  \text{ and }  Y \indep_\sigma I_X \mid X,\mathbf{Z}$. We first show that $\mathbf{Z}$ contains no descendants of $X$. Suppose, for contradiction, that there exists $Z_i \in \mathbf{Z} \cap \mathrm{De}(X)$. Since $I_X$ has directed edges into $X$, and $Z_i$ is a descendant of $X$, there exists a directed path from $I_X$ to $Z_i$. Such a directed path is $\sigma$-open without conditioning, contradicting $\mathbf{Z} \indep_\sigma I_X$. Hence
$\mathbf{Z} \cap \mathrm{De}(X)=\emptyset$.

It remains to show that $\mathbf{Z}$ $\sigma$-blocks every backdoor path from $X$ to $Y$. Suppose that it is not the case. Then there exists a backdoor path from $X$ to $Y$ that is $\sigma$-open given $\mathbf{Z}$. Prepending the edge $I_X \to X$ to this path gives a path from $I_X$ to $Y$. Since the path enters $X$ through a backdoor edge, conditioning on $X$ does not block this path; rather, the path remains $\sigma$-open given $X,\mathbf{Z}$. Therefore,
$Y \not\indep_\sigma I_X \mid X,\mathbf{Z},$ contradicting the second assumed condition. Thus $\mathbf{Z}$ contains no descendants of $X$ and blocks every backdoor path from $X$ to $Y$. Therefore, $\mathbf{Z}$ satisfies Definition~\ref{def:backdoor-adjustment}. This proves the equivalence.
\end{proof}

\subsection{Lemma \ref{lem:existence-adjustement}}

\begin{proof}
Let $G^{\mathrm{acy}}$ be a $\sigma$-acyclification of $G$. 
By Lemma~\ref{lem:equiv-backdoor} and the equivalence of $\sigma$-separation and d-separation under $\sigma$-acyclification, a set $\mathbf{Z}\subseteq \mathbf{O}$ is a backdoor adjustment set for $(X, Y)$ in $G$ if and only if it is a backdoor adjustment set for $(X, Y)$ in $G^{\mathrm{acy}}$. Hence
\[
\exists\, \mathbf{Z} \subseteq \mathbf{O} \text{ backdoor adjustment in } G
\;\Longleftrightarrow\;
\exists\, \mathbf{Z} \subseteq \mathbf{O} \text{ backdoor adjustment in } G^{\mathrm{acy}}.
\]
Also, using the fact that $X$ is a singleton and $G^{\mathrm{acy}}$ is a DAG, we have by \cite{perkovic2018complete} that
\[
\mathbf{Z} \text{ backdoor adjustment in } G^{\mathrm{acy}}
\;\Longleftrightarrow\;
\mathbf{Z} \text{ adjustment in } G^{\mathrm{acy}}.
\]
In addition, since \(G^{\mathrm{acy}}\) is a DAG, let $\mathcal{M_\mathbf{O}}$ be the MAG induced by $G^{\mathrm{acy}}$ over the observed variables $\mathbf{O}$. By \cite{zhang2008causal}, $\mathcal{M_\mathbf{O}}$  preserves all conditional independence relations among observed variables from $G^{\mathrm{acy}}$: for any $\mathbf{Z} \subseteq \mathbf{O}$, $X$, $Y$ are $m$-separated given $\mathbf{Z}$ in $\mathcal{M_\mathbf{O}}$ if and only if they are d-separated given $\mathbf{Z}$ in $G^{\mathrm{acy}}$. Therefore, for observed adjustment sets, validity in the MAG is equivalent to validity in the DAG. 
\[
\exists\, \mathbf{Z} \subseteq \mathbf{O} \text{ adjustment in } G^{\mathrm{acy}}
\;\Longleftrightarrow\;
\exists\, \mathbf{Z} \subseteq \mathbf{O} \text{ adjustment in } \mathcal{M_\mathbf{O}}.
\]
Previous results for MAG $\mathcal{M_\mathbf{O}}$ \cite{li2024local} state that a valid set exists if and only if there exists one contained in the Markov blanket of $Y$ (excluding $X$) in $\mathcal{M_\mathbf{O}}$ 
\[
\exists\, \mathbf{Z} \subseteq \mathbf{O} \text{ adjustment in } \mathcal{M_\mathbf{O}}
\;\Longleftrightarrow\;
\exists\, \mathbf{Z} \subseteq \mathrm{MB}_m^{\mathcal{M_\mathbf{O}}}(Y)\setminus\{X\}.
\]
By the Markov blanket equivalence between DAGs and MAGs from Lemma \ref{lem:eqmbmag}:
\[
\exists\, \mathbf{Z} \subseteq \mathrm{MB}_m^{\mathcal{M_\mathbf{O}}}(Y)\setminus\{X\}
\;\Longleftrightarrow\;
\exists\, \mathbf{Z} \subseteq \mathrm{MB}_d^{G^{\mathrm{acy}}}(Y)\setminus\{X\}.
\]
By the equivalence of Markov blankets 
from Lemma \ref{lem:mb-acy}, we have $
\mathrm{MB}^{G}_{\sigma}(Y) = \mathrm{MB}^{G^{\mathrm{acy}}}_d(Y)$.
Substituting yields
\[
\exists\, \mathbf{Z} \subseteq \mathrm{MB}^{G^{\mathrm{acy}}}_d(Y)\setminus\{X\}
\;\Longleftrightarrow\;
\exists\, \mathbf{Z} \subseteq \mathrm{MB}^G_{\sigma}(Y)\setminus\{X\}.
\]

Combining the above equivalences completes the proof.
\end{proof}

\subsection{Lemma \ref{lem:witness-invariance}}
\begin{proof}
We show that the following two statements hold:
\begin{enumerate}
    \item $W \in \mathrm{MB}^G_\sigma(X)$ in $G$ if and only if $W \in MB^{G^{\mathrm{acy}}}_d(X)$ in $G^{\mathrm{acy}}$.
    \item For every $\mathbf{Z} \subseteq \mathbf{O} \setminus \{X,Y\}$,
    \[
    \begin{aligned}
    W \not\!\indep_\sigma Y \mid \mathbf{Z}
    &\Longleftrightarrow
    W \not\!\indep_d Y \mid \mathbf{Z}, \\
    W \indep_\sigma Y \mid \mathbf{Z} \cup \{X\}
    &\Longleftrightarrow
    W \indep_d Y \mid \mathbf{Z} \cup \{X\},
    \end{aligned}
    \]
    where the right-hand side is evaluated in $G^{\mathrm{acy}}$.
\end{enumerate}

Statement (1) is exactly the Markov blanket equivalence under $\sigma$-acyclification (Lemma \ref{lem:mb-acy}).

For (2), let $W \in \mathbf{O} \setminus \{X,Y\}$ and $\mathbf{Z} \subseteq \mathbf{O} \setminus \{X,Y\}$. By the equivalence of $\sigma$-separation in $G$ and $d$-separation in $G^{\mathrm{acy}}$, we have
\[
W \indep_\sigma Y \mid \mathbf{Z}
\;\Longleftrightarrow\;
W \indep_d Y \mid \mathbf{Z},
\qquad
W \indep_\sigma Y \mid \mathbf{Z} \cup \{X\}
\;\Longleftrightarrow\;
W \indep_d Y \mid \mathbf{Z} \cup \{X\}.
\]
Taking negations yields the corresponding dependence statements, proving (2).

Then the claim follows immediately: by (1), the candidate witnesses $W$ lie in the same Markov blanket in both graphs, and by (2), the R1/R2 conditions are preserved. Hence existence is invariant under $\sigma$-acyclification.
\end{proof}

\subsection{Lemma \ref{theom:soundcomp}}
\begin{proof}
    Soundness: Assume that the conditions of R1 or R2 hold in $G$. By Lemma \ref{lem:witness-invariance}, the same conditions hold in any $G^{\mathrm{acy}}$. Since $G^{\mathrm{acy}}$ is acyclic, and a candidate adjustment set $\mathbf{Z}$ is contained in the observed $\mathrm{MB}^{G^{\mathrm{acy}}}_d(Y) \setminus \{X\}$, Theorem 2 in \cite{entner2013data} implies that R1 and R2 are sound in $G^{\mathrm{acy}}$.
    
    If R1 applies in $G^{\mathrm{acy}}$, it concludes that: (i) there is a direct edge $X \rightarrow Y$ in $G^{\mathrm{acy}}$; (ii) $\mathbf{Z}$ is a valid adjustment set for estimating the effect of $X$ on $Y$ in $G^{\mathrm{acy}}$. By Lemma \ref{lem:edge-preserved}, the edge $X \rightarrow Y$ in $G^{\mathrm{acy}}$ implies the presence 
    in $G$. By Lemma \ref{lem:equiv-backdoor} and \ref{lem:acy-transfer-sep}, the set $Z$ is also a valid adjustment set in $G$. Therefore, both conclusions of R1 are sound in the original cyclic graph $G$.
    
    If R2 applies in $G^{\mathrm{acy}}$, it concludes that there is no direct edge in $G^{\mathrm{acy}}$. By Lemma \ref{lem:edge-preserved}, this absence is preserved in $G$. Hence, R2 is also sound in $G$.
    
    Therefore, R1 and R2 are sound in simple cyclic SCMs.

    Completeness: Suppose that neither R1 nor R2 applies in $G$. By Lemma \ref{lem:witness-invariance}, the same conditions over $\mathbf O$ fail in $G^{\mathrm{acy}}$.  Let $\mathcal M_{\mathbf O}$ be the MAG over $\mathbf O$ obtained by latent projection of
    $G^{\mathrm{acy}}$. By Lemma~\ref{lem:eqmbmag}, these observed conditional independence relations are exactly the $m$-separation relations in
    \(\mathcal M_{\mathbf O}\).

    Therefore, neither R1 nor R2 applies in $\mathcal M_{\mathbf O}$. By the
    completeness result from \cite{li2024local}, no criterion based only on
    the observed conditional independence relations can determine whether the
    effect of $X$ on $Y$ is identifiable. Since $G^{\mathrm{acy}}$ and
    $\mathcal M_{\mathbf O}$ induce the same observed conditional independence
    assertions by \cite{zhang2008causal}, the same holds for $G^{\mathrm{acy}}$ and $G$ by Lemma~\ref{lem:acy-transfer-sep}. Therefore, the same non-identifiabilityfrom conditional independence information holds for the original possibly cyclic SCM $G$.  In conclusion, any ambiguity in determining the presence or absence of a direct causal effect from $X$ to $Y$ in $G^{\mathrm{acy}}$ based solely on conditional independencies also holds in $G$.
\end{proof}

\begin{figure}[h]
    \centering
        \includegraphics[width = 0.33\linewidth]{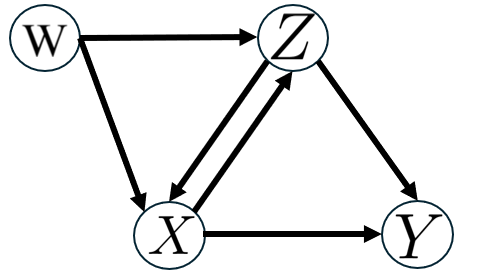}
    \caption{Example graph $G$ used in the proof of Lemma~\ref{lem:relaxed_pre-treatment}.}
    \label{fig:lem_pre}
\end{figure}

\subsection{Lemma \ref{lem:relaxed_pre-treatment}}
\begin{proof}
     We show $R1$ is no longer sound by providing an example. Consider the simple cyclic SCM $G$ in Figure~\ref{fig:lem_pre} such that $Y\notin\mathrm{SCC}(X)$ and $|\mathrm{SCC(X)}|=2$. Then since $Z\in \operatorname{De}(X)$, we have that $Z\not\indep_\sigma I_{X}$ in $G^{I_X}$ and thus that $Z$ cannot be in a valid backdoor adjustment set by Corollary~8.4 from \cite{forre2020causal}. However, for $G$, we have that $W \in \mathrm{MB}^G_{\sigma}(X)\setminus\{Y\}$ and  $\{Z\} \subseteq\mathrm{MB}^G_{\sigma}(Y)\setminus\{X,W\}$, while 
     \[
    W \not\indep_\sigma Y \mid Z
    \quad\text{and}\quad
    W \indep_\sigma Y \mid Z \cup \{X\}.
    \]
     and thus by $\sigma$-faithfulness (Assumption~$(ii)$)
    \[
    W \not\indep Y \mid Z
    \quad\text{and}\quad
    W \indep Y \mid Z \cup \{X\}.
    \]
    Hence, $R1$ will return $\{Z\}$ as a valid backdoor adjustment set. Therefore, $R1$ is not sound.

    Finally, we note that adding any node to the $\mathrm{SCC}(X)$ that is not connected to either $W$ or $Y$ will not change these conditional independence relations. Hence, we can always construct a simple SCM $G'$ such that $Y\notin\mathrm{SCC}(X)$ and $|\mathrm{SCC(X)}|\geq2$ for which $R1$ is not sound.
\end{proof}

\subsection{Lemma \ref{thm:mb-algorithms-cyclic}}
\begin{proof}
Let $\mathcal{A}$ be any sound Markov blanket algorithm. Since $\mathcal{A}$ only applies conditional tests under observed variables, the behavior is determined entirely by the observed independence assertions over $\mathbf{O}$. By Lemma \ref{lem:acy-transfer-sep}, every conditional independence query has the same truth value under the observational distribution generated by $G$ as implied by the d-separation in $G^{\mathrm{acy}}$. Since $G^{\mathrm{acy}}$ may contain latent variables, let $\mathcal M_{\mathbf O}$ be its latent
projection MAG over $\mathbf{O}$. Thus by soundness and completeness of $\mathcal A$ for the observed CI relations, $\mathcal A$ returns $\mathrm{MB}^{\mathcal M_{\mathbf O}}_m(X)$.
By Lemma \ref{lem:eqmbmag}, the Markov blanket in $G^{\mathrm{acy}}$, i.e., $\mathrm{MB}^{G^{\mathrm{acy}}}_d(X)$ coincides with the $m$-Markov blanket in $\mathcal M_{\mathbf O}$, i.e., $\mathrm{MB}^{\mathcal M_{\mathbf O}}_m(X)$. 
By Lemma~\ref{lem:mb-acy}, $\mathrm{MB}^{G^{\mathrm{acy}}}_d(X) = \mathrm{MB}^G_\sigma(X)$, which completes the proof.
\end{proof}

\section{Data generation}\label{ap:datagen}

Let's consider simple SCMs where $\mathbf{V}$ are random variables, $W$ is the weighted adjacency matrix consistent with the underlying graph structure and $\mathbf{U}$ is the vector of independent noise variables. Interventions are modeled using the do-operator. 

\subsection{Linear SCMs}

We define a linear SCM:

\[
\mathbf{V} := W\mathbf{V} + \mathbf{U}.
\]

In acyclic graphs, $W$ is strictly triangular, and the system admits a unique solution. In the presence of cycles, we need to enforce that $(I − W)$ is invertible during data generation, which guarantees a unique solution (simple SCMs), s.t. $\mathbf{V} = (I - W)^{-1} \mathbf{U}$. In the code, we verify that $I - W$ is invertible to ensure the existence of the matrix $(I - W)^{-1}$.

\subsubsection{Interventions}

Let $X$ be the treatment variable and $Y$ the outcome variable. Under an intervention $do(X = x)$, the structural equation for $X$ is replaced by a constant:

\[
\begin{split}
    \mathbf{V}_K &= W_{KK}\mathbf{V}_K + xW_{KX} + \mathbf{U}_K\\
    (I-W_{KK})\mathbf{V}_K &= xW_{KX} + \mathbf{U}_K\\
    \mathbf{V}_K &= (I-W_{KK})^{-1}(xW_{KX} + \mathbf{U}_K)
\end{split}
\]

where $\mathbf{K} = \mathbf{V} \setminus \{X\}$.

\subsubsection{Causal effect} \label{ap:ce}

The total causal effect (CE) from $X$ on $Y$ can be defined in both a binary and a continuous setting. In the setting where the treatment is continuous:
\[
    CE=\frac{\partial}{\partial x}E^{\operatorname{do}(X:=x)}[Y],
\]
$E^{\operatorname{do}(\mathbf{X}:=x)}$ indicates the distribution where the do action assignment has been performed.

\paragraph{Acyclic case} In the acyclicity setting, when considering the pre-treatment assumption, the causal effect of $X$ on $Y$ is:

\[
\begin{split}
        CE&:=\frac{\partial}{\partial x}E^{\operatorname{do}(X:=x)}[Y]\\
        &\overset{}{=}\frac{\partial}{\partial x}E[(W_{KK}\mathbf{V}_K+xW_{KX}+\mathbf{U}_K)e_Y]\\
        &\overset{}{=}\frac{\partial}{\partial x}E[xW_{KX}e_Y]+\frac{\partial}{\partial x}E[(W_{KK}\mathbf{V}_K+\mathbf{U}_K)e_Y]\\
        &\overset{}{=}\frac{\partial}{\partial x}xW_{KX}e_Y\\
        &=W_{KX}e_Y=W_{YX},
    \end{split}
\]

Here we use linearity of expectation and the fact that, in DAGs, $\mathbf{V}_K$ does not depend on $x$ through feedback and in the pre-treatment setting

\paragraph{Cyclic case} In the cyclic setting, feedback induces dependence on $x$ through $\mathbf{V}_k$. Using the same form above, we obtain:

\[
    \begin{split}
        CE&=\frac{\partial}{\partial x}E^{\operatorname{do}(X:=x)}[Y]\\
        &\overset{}{=}\frac{\partial}{\partial x}E[(I-W_{KK})^{-1}(xW_{KX}+\mathbf{U})e_Y]\\
        &\overset{}{=}\frac{\partial}{\partial x}E[x(I-W_{KK})^{-1}W_{KX} e_y]+\frac{\partial}{\partial x}E[(I-W_{KK})^{-1}\mathbf{U}_Y]\\
        &\overset{}{=}\frac{\partial}{\partial x}x(I-W_{KK})^{-1}W_{KX} e_y\\
        &=(I-W_{KK})^{-1}W_{KX} e_y,
    \end{split}
\]

\subsection{Nonlinear SCMs}

 We define a nonlinear SCM as:
 
 \[
 \mathbf{V} := tanh(W\mathbf{V}) + \mathbf{U}
 \]

where $tanh$ is applied componentwise. In acyclic graphs, the system can be solved by forward substitution. In cyclic graphs, we enforce a condition on the operator norm of W.

\begin{lemma}
If $\|W\| < 1$, where $\|\cdot\|$ is an induced operator norm, then the equation
\begin{equation}
    \mathbf{V} = \tanh(W\mathbf{V}) + \mathbf{U}
\end{equation}
admits a unique solution.
\end{lemma}

\begin{proof}
    Define $F:\mathbb{R}^d \to \mathbb{R}^d$ as
    \[
    F(\mathbf{V}) := \tanh(W\mathbf{V}) + \mathbf{U}.
    \]
    For any $\mathbf{V}, Y \in \mathbb{R}^d$,
    \[
    \|F(\mathbf{V}) - F(Y)\|
    = \|\tanh(W\mathbf{V}) - \tanh(WY)\|.
    \]
    Since $\tanh$ is componentwise $1$-Lipschitz,
    \[
    \|\tanh(W\mathbf{V}) - \tanh(WY)\| \le \|W\mathbf{V} - WY\|.
    \]
    Using the definition of the induced operator norm,
    \[
    \|W\mathbf{V} - WY\| \le \|W\| \, \|\mathbf{V} - Y\|.
    \]
    Therefore,
    \[
    \|F(\mathbf{V}) - F(Y)\| \le \|W\| \, \|\mathbf{V} - Y\|.
    \]
    If $\|W\| < 1$, then $F$ is a contraction. Since $\mathbb{R}^d$ is complete, the Banach fixed-point theorem guarantees the existence and uniqueness of a fixed point $\mathbf{V}^\ast$ such that
    \[
    \mathbf{V}^\ast = \tanh(W\mathbf{V}^\ast) + \mathbf{U}.
    \]
\end{proof}

\paragraph{Implementation details.}
In practice, we enforce the condition $\|W\| < 1$ using the Euclidean operator norm, i.e., the largest singular value of $W$.

\subsubsection{Causal Effect}

In contrast to the linear case, the causal effect in non-linear SCMs does not admit a simple closed-form expression. In particular, the quantity $\frac{\partial}{\partial x} \mathbb{E}^{\operatorname{do}(X := x)}[Y]$ generally depends on the intervention value $x$, and therefore does not reduce to a constant effect.

\subsection{Generation and assumptions translations}

\paragraph{Latent and observed variables} To simulate partial observability, we randomly select a subset of nodes as observed variables. The final dataset contains only samples corresponding to these observed nodes, while always including the designated treatment and outcome variables. This ensures that causal queries remain well-defined.

\paragraph{Treatment and outcome} For each generated SCM, we fix a treatment variable $X$ and an outcome variable $Y$. We enforce the presence (or absence) of a direct edge from $X$ to $Y$.

\section{Experimental setup}

\paragraph{Conditional independence tests}
In all experiments, conditional independence tests were performed using Fisher’s Z-test with significance level $\alpha = 0.01$.

\paragraph{Hardware}
We ran all experiments on a machine with 16 Intel Xeon E5-2630 v3 cores at 2.40 GHz, supporting 32 hardware threads, and 512 GB of RAM available. No GPU was used.

\subsection{Compute time}

\begin{figure}[h]
    \centering
    \begin{subfigure}{0.3\textwidth}
        \centering
        \includegraphics[width = \linewidth]{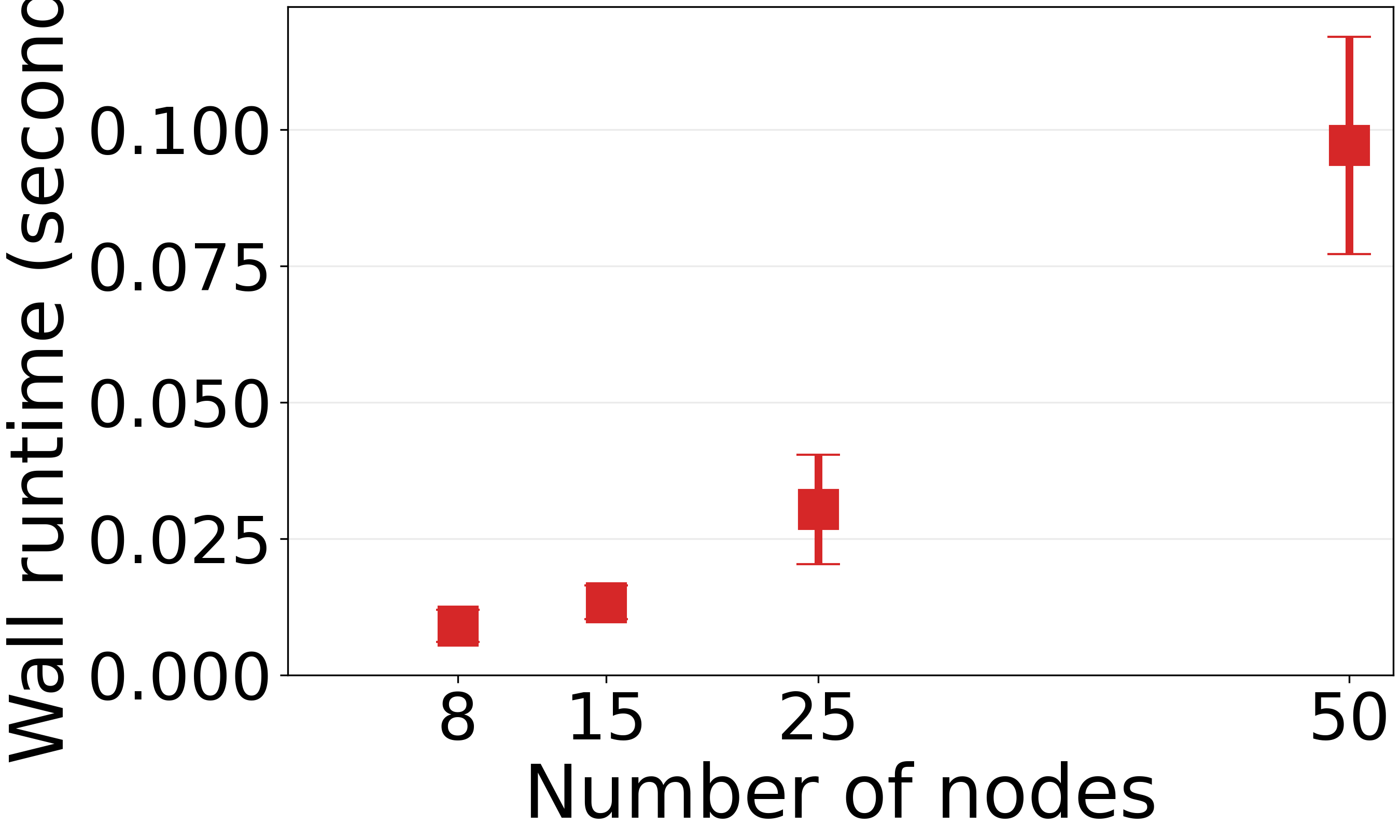}
        \caption{Acyclic graphs}
        \label{acy_runtime_precision}
    \end{subfigure}
    \begin{subfigure}{0.3\textwidth}
        \centering
        \includegraphics[width = \linewidth]{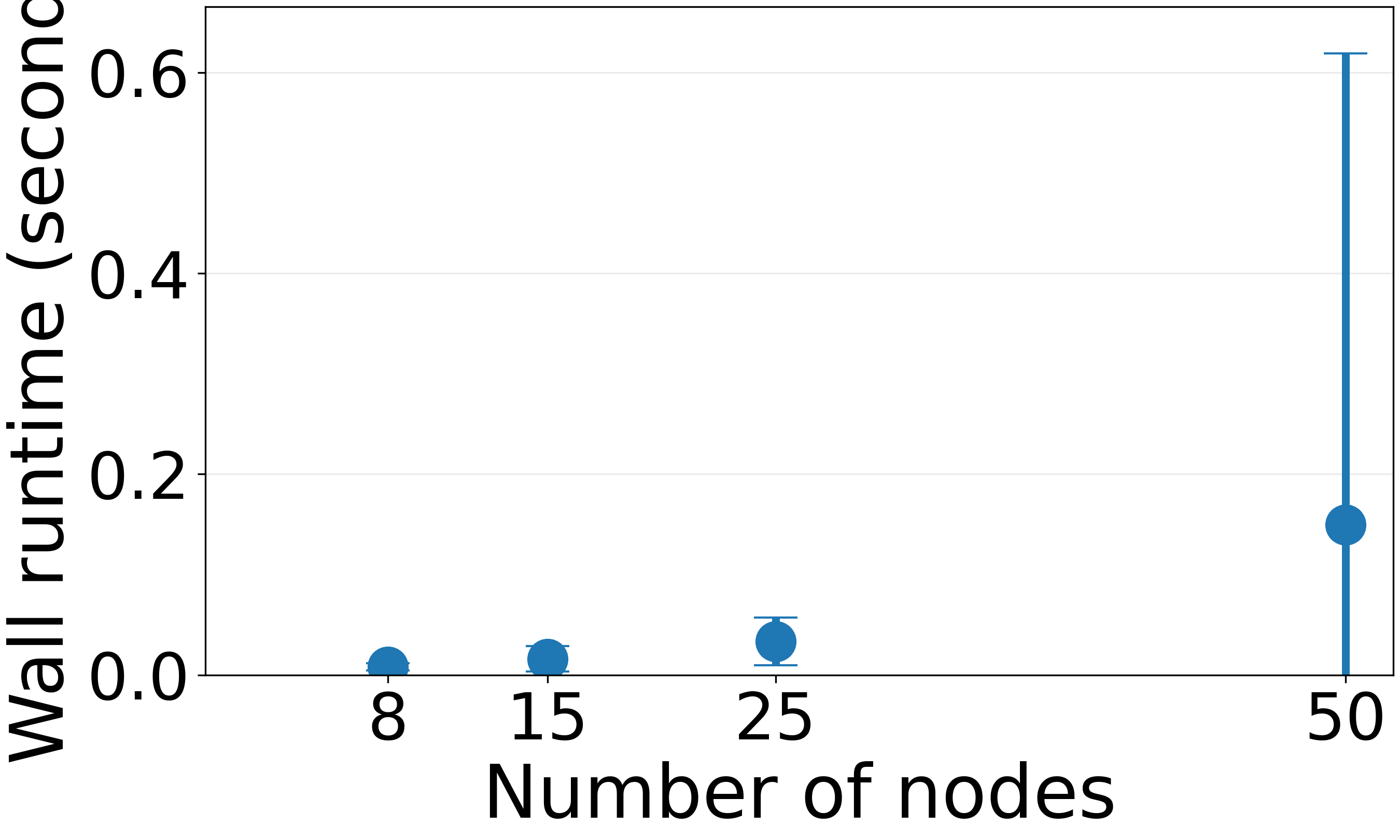}
        \caption{Cyclic graphs}
        \label{cic_runtime_precision}
    \end{subfigure}
    \caption{Wall runtime in seconds (s) of the adjustment set discovery method under the pre-treatment assumption for increasing graph sizes when computing the precision and using TC for the Markov blanket discovery. The x-axis represents the number of nodes in the graph, while the y-axis shows the measured wall runtime. Subfigures \ref{acy_runtime_precision} and \ref{cic_runtime_precision} are for acyclic and cyclic settings, respectively.}
    \label{fig:wall_runtime_pre_precision}
\end{figure}

\begin{figure}[h]
    \centering
    \begin{subfigure}{0.3\textwidth}
        \centering
        \includegraphics[width = \linewidth]{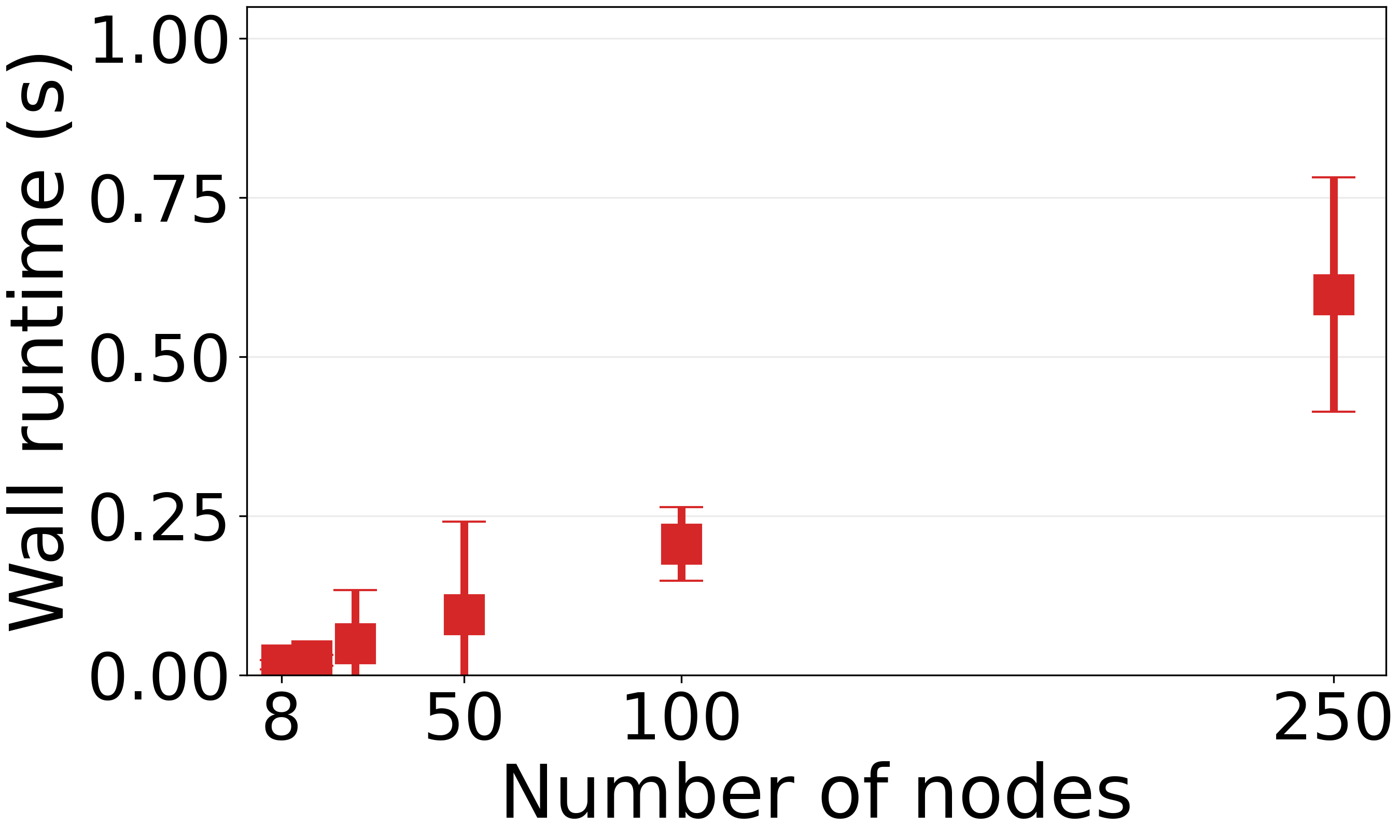}
        \caption{Acyclic graphs}
        \label{acy_runtime}
    \end{subfigure}
    \begin{subfigure}{0.3\textwidth}
        \centering
        \includegraphics[width = \linewidth]{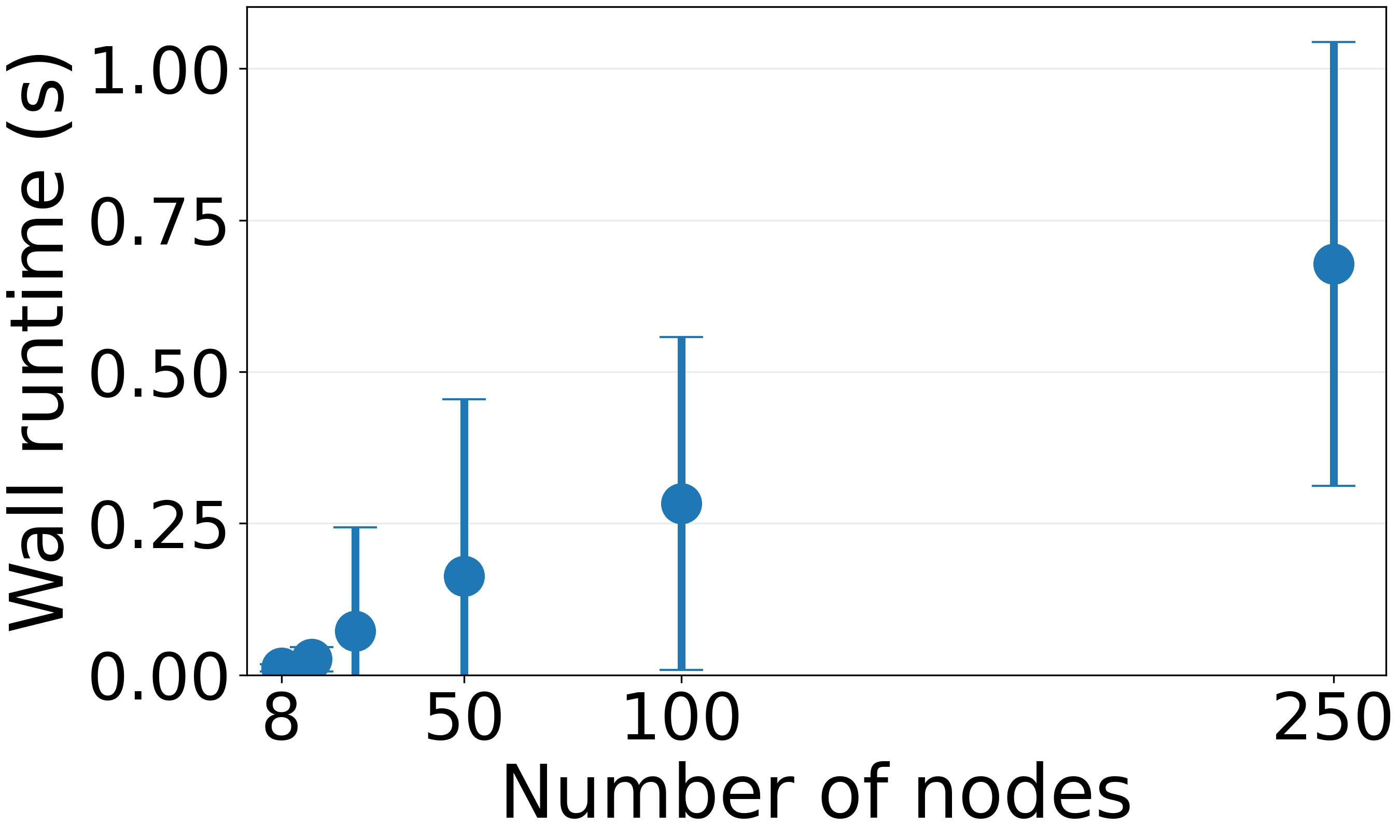}
        \caption{Cyclic graphs}
        \label{cic_runtime}
    \end{subfigure}
    \caption{Wall runtime in seconds (s) of the adjustment set discovery method under the pre-treatment assumption for increasing graph sizes without computing the precision and using Fast-IAMB for the Markov blanket discovery. The x-axis represents the number of nodes in the graph, while the y-axis shows the measured wall runtime. Subfigures \ref{acy_runtime} and \ref{cic_runtime} are for acyclic and cyclic settings, respectively.}
    \label{fig:wall_runtime_pre}
\end{figure}

Figure \ref{fig:wall_runtime_pre_precision} shows that the wall runtime of the adjustment set discovery method under the pre-treatment assumption for graphs of varying sizes, considering both acyclic and cyclic settings when computing the precision and using TC for the Markov blanket discovery. As expected, the runtime generally increases as the graph size grows. Larger graphs exhibit higher runtimes due to the increased computational complexity. Comparing the two configurations, cyclic graphs consistently show slightly higher runtimes than acyclic graphs.

Figure \ref{fig:wall_runtime_pre} shows the same setup, but without computing the precision and using Fast-IAMB instead of TC for the Markov blanket discovery. As expected, the runtime generally increases as the graph size grows. Larger graphs exhibit higher runtimes due to the increased computational complexity. Comparing the two settings, cyclic graphs consistently show slightly higher runtimes than acyclic graphs.

Comparing the smaller graph sizes in Figures~\ref{fig:wall_runtime_pre} and~\ref{fig:wall_runtime_pre_precision} shows that omitting the precision computation yields a substantial speedup, especially for the cyclic graph with 50 nodes. This is expected, since computing precision requires checking whether a set is a valid adjustment set, which entails verifying the $\sigma$-backdoor criterion over a combinatorially growing number of paths.

\section{Extension of the results}
\subsection{Large graphs} \label{ap:largegraphs}
Figure \ref{fig:largergraphs} shows the behavior of the evaluated metrics as a function of the number of nodes. The graphs were created under the same experimental setups as described in Section~\ref{sec:experimental_results} with the important change that the possible size of the set $|\mathbf{Z}|\leq10$ in checking rule $R1$ and $R2$ (Definitions~\ref {def:R1}, \ref{def:R2}). Overall, there is a clear trend of improving performance with increasing graph size. Consistent with the results in Section~\ref{sec:experimental_results}, cyclic models show better performance than acyclic ones across both metrics and graph sizes. This seems to suggest that even when enforcing a maximum size on the set $\mathbf{Z}$, the algorithm works well. This could be due to the algorithm correctly rejecting scenarios where it cannot decide based on the smaller set, and/or the valid adjustment sets being small in general in our graphs.

\begin{figure}[h]
    \centering
    \begin{subfigure}{0.3\textwidth}
        \centering
        \includegraphics[width=\linewidth]{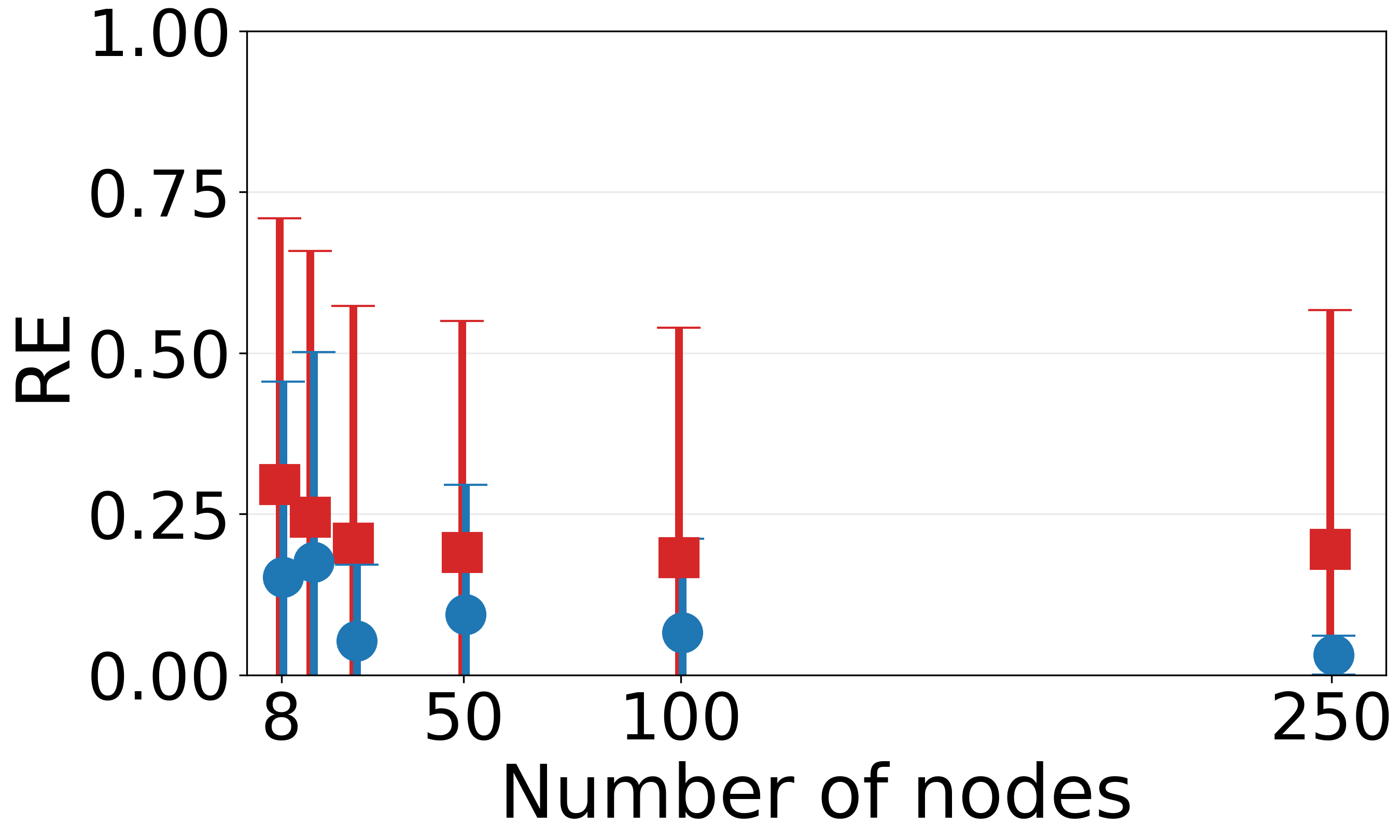}
        \caption{Relative error for linear models}
    \end{subfigure}
    \begin{subfigure}{0.3\textwidth}
        \centering
        \includegraphics[width=\linewidth]{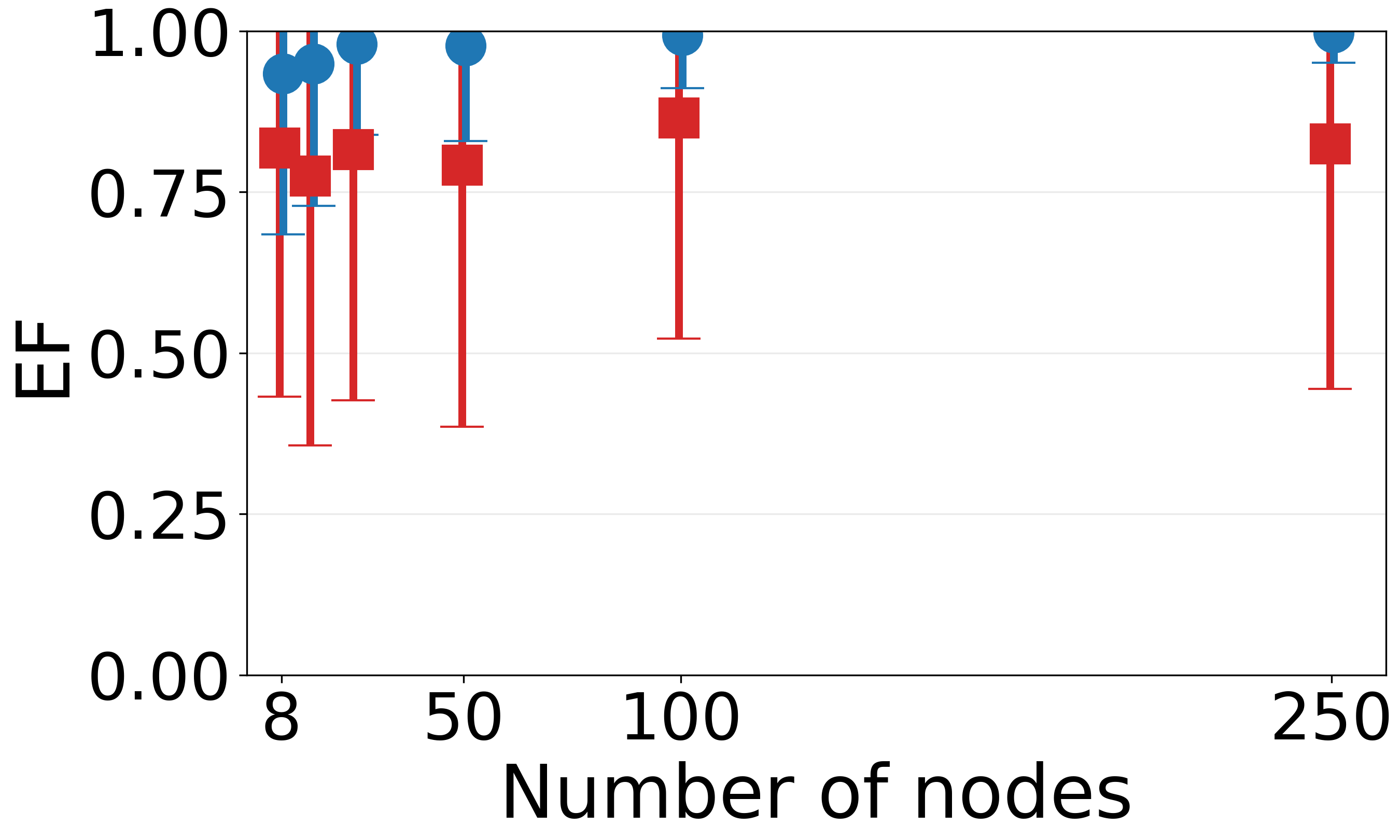}
        \caption{Edge fraction acc. for (non-)linear models}
    \end{subfigure}
    \caption{Performance across varying graph sizes. The x-axis represents the number of nodes, while the y-axis shows the evaluated metric. Results are reported for acyclic (red squares) and cyclic (blue circles) configurations. Markers indicate mean values, and error bars denote one standard deviation. Each subfigure corresponds to a different performance metric.}
    \label{fig:largergraphs}
\end{figure}


\subsection{Violation of the pre-treatment assumption} \label{ap:nonpre}

\begin{figure}[h]
    \centering
    \begin{subfigure}{0.24\textwidth}
        \centering
        \includegraphics[width=\linewidth]{figs/Precision_pre/p_prec_samples_vs_graph_type_8n_tc_xy_pre.png}
        \caption{8 nodes pre.}
        \label{figap:pre_8n}
    \end{subfigure}
    \begin{subfigure}{0.24\textwidth}
        \centering
        \includegraphics[width=\linewidth]{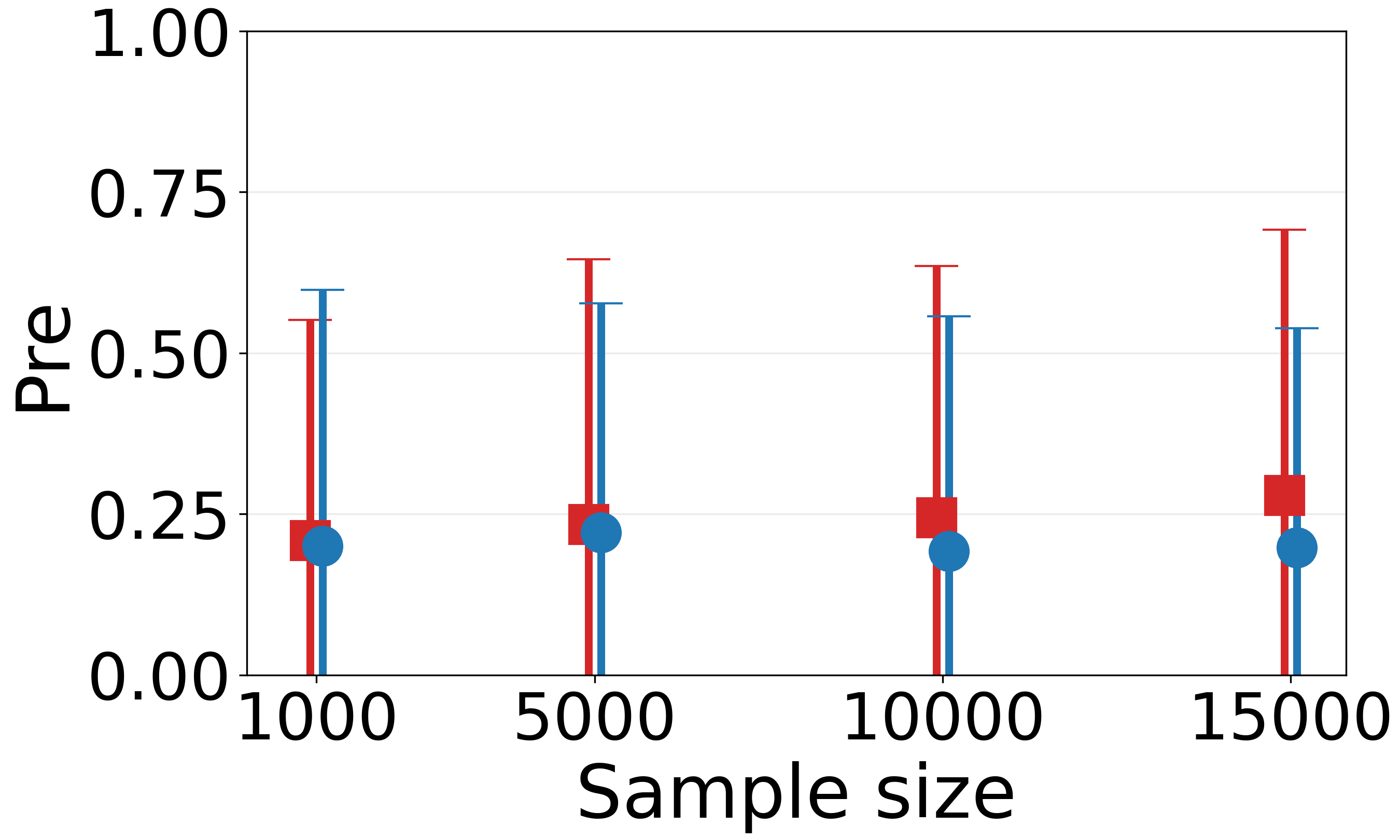}
        \caption{8 nodes non-pre.}     
        \label{figap:nonpre_8n}
    \end{subfigure}
    \begin{subfigure}{0.24\textwidth}
        \centering
        \includegraphics[width=\linewidth]{figs/Precision_pre/p_prec_samples_vs_graph_type_25n_tc_xy_pre.png}
        \caption{25 nodes pre.}   
        \label{figap:pre_25n}
    \end{subfigure}
    \begin{subfigure}{0.24\textwidth}
        \centering
        \includegraphics[width=\linewidth]{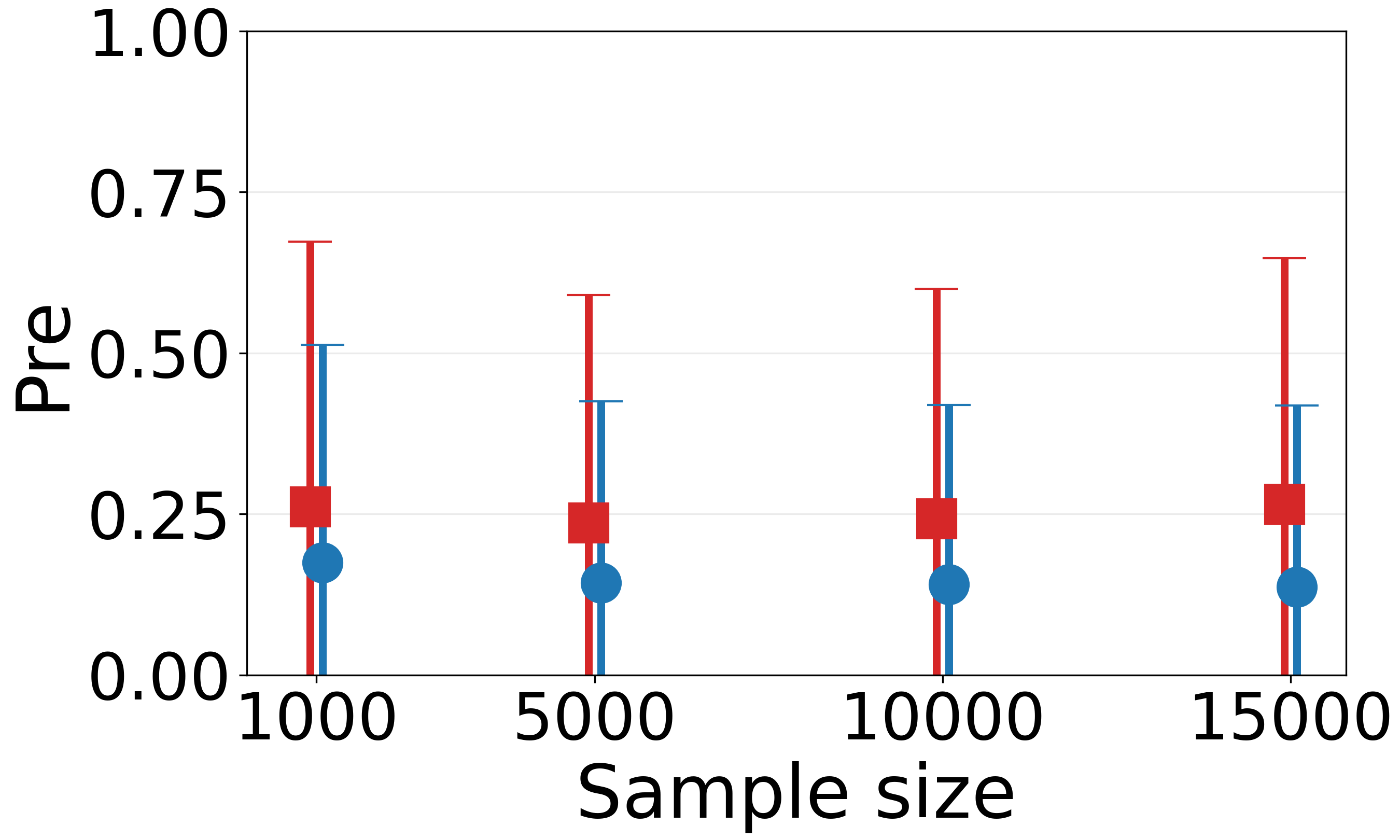}
        \caption{25 nodes non-pre.}
        \label{figap:nonpre_25n}
    \end{subfigure}
    \caption{Precision of linear models in acyclic (red square) and cyclic (blue circle) settings across sample sizes, comparing scenarios that violate the pre-treatment assumption (non-pre.) with those that satisfy it (pre.). Markers show the mean, with error bars indicating one standard deviation. Subfigures correspond to graph sets with different numbers of nodes and pre-treatment assumptions. }
    \label{fig:RE_smallgraphs_nonpre}
\end{figure}

Figure~\ref{fig:RE_smallgraphs_nonpre} shows the precision of linear models in acyclic (red square) and cyclic (blue circle) settings across sample sizes, comparing scenarios that violate the pre-treatment assumption (non-pre.) with those that satisfy it (pre.). Markers show the mean, with error bars indicating one standard deviation. When looking at the effect of the violation of the pre-treatment assumption for 8 nodes by comparing Figures \ref{figap:pre_8n} and \ref{figap:nonpre_8n}, we see that the acyclic SCM reduces in precision by the violation more than the cyclic configuration. Moreover, in the pre-treatment violation, both configuration types are performing equally. This is in contrast with the better performance of the acyclic configuration when looking at Figure~\ref{fig:nonlinear_smallgraphs}. The same pattern repeats itself for the other graphs, see for example, the 25 nodes in Figures \ref{figap:pre_25n} and \ref{figap:nonpre_25n}.

\subsection{Empty fraction} \label{ap:emptyfrac}

Figure \ref{fig:empty_frac_small_graph} shows the empty fraction, i.e. the percentage where the algorithm returns an empty set, for linear and non-linear models together across acyclic (red square) and cyclic (blue circle) settings across sample sizes.  Markers show the mean, with error bars indicating one standard deviation. Subfigures correspond to graph sets with different numbers of nodes. Overall, cyclic settings consistently exhibit lower empty fraction, remaing bellow $20\%$. Acyclic cases show slightly higher empty fraction, typically around $25\%$, with the exception in the $15$ nodes settings, where it increases substantially, being around $50\%$.

\begin{figure}[h]
    \centering
    \begin{subfigure}{0.24\textwidth}
        \centering
        \includegraphics[width=\linewidth]{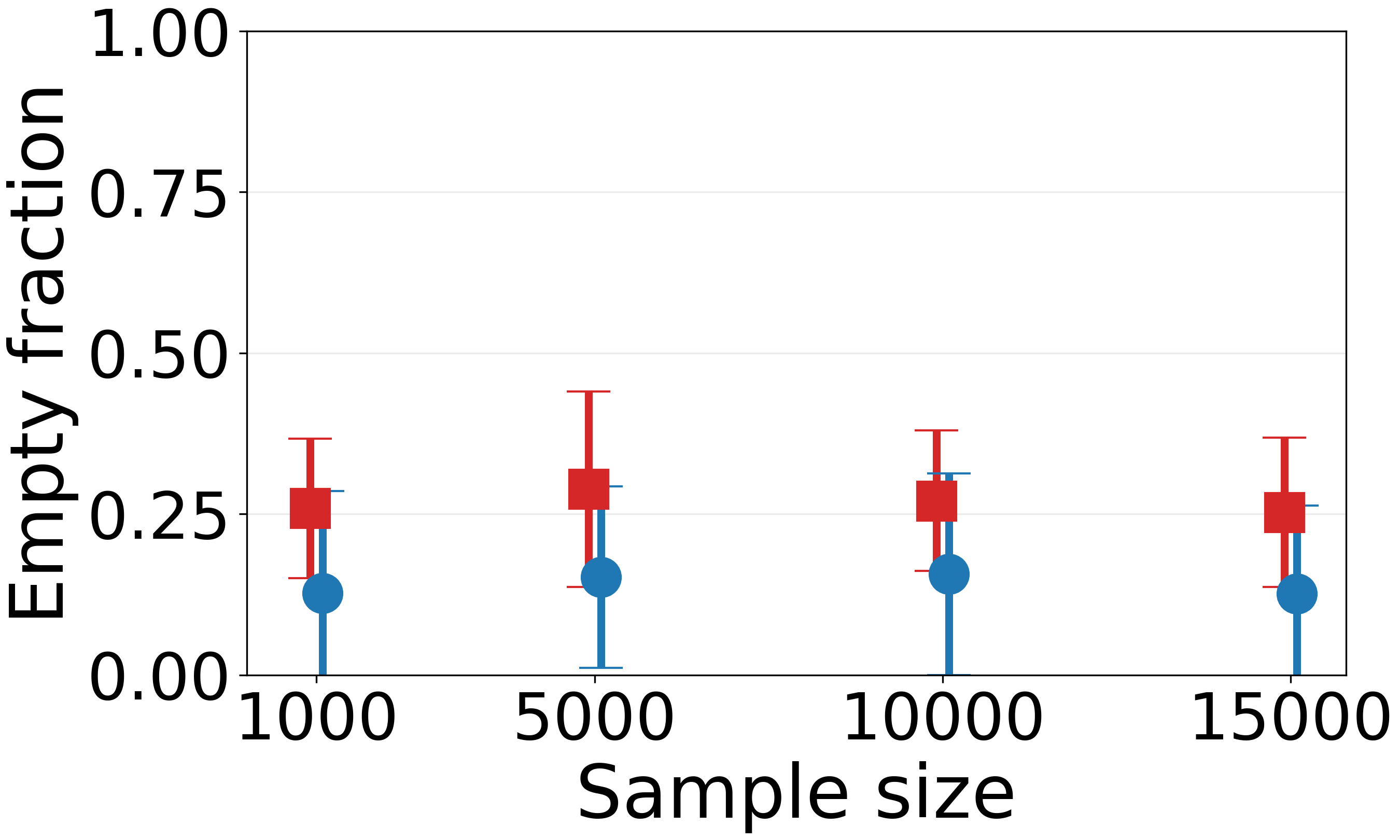}
        \caption{8 nodes}
    \end{subfigure}
    \begin{subfigure}{0.24\textwidth}
        \centering
        \includegraphics[width=\linewidth]{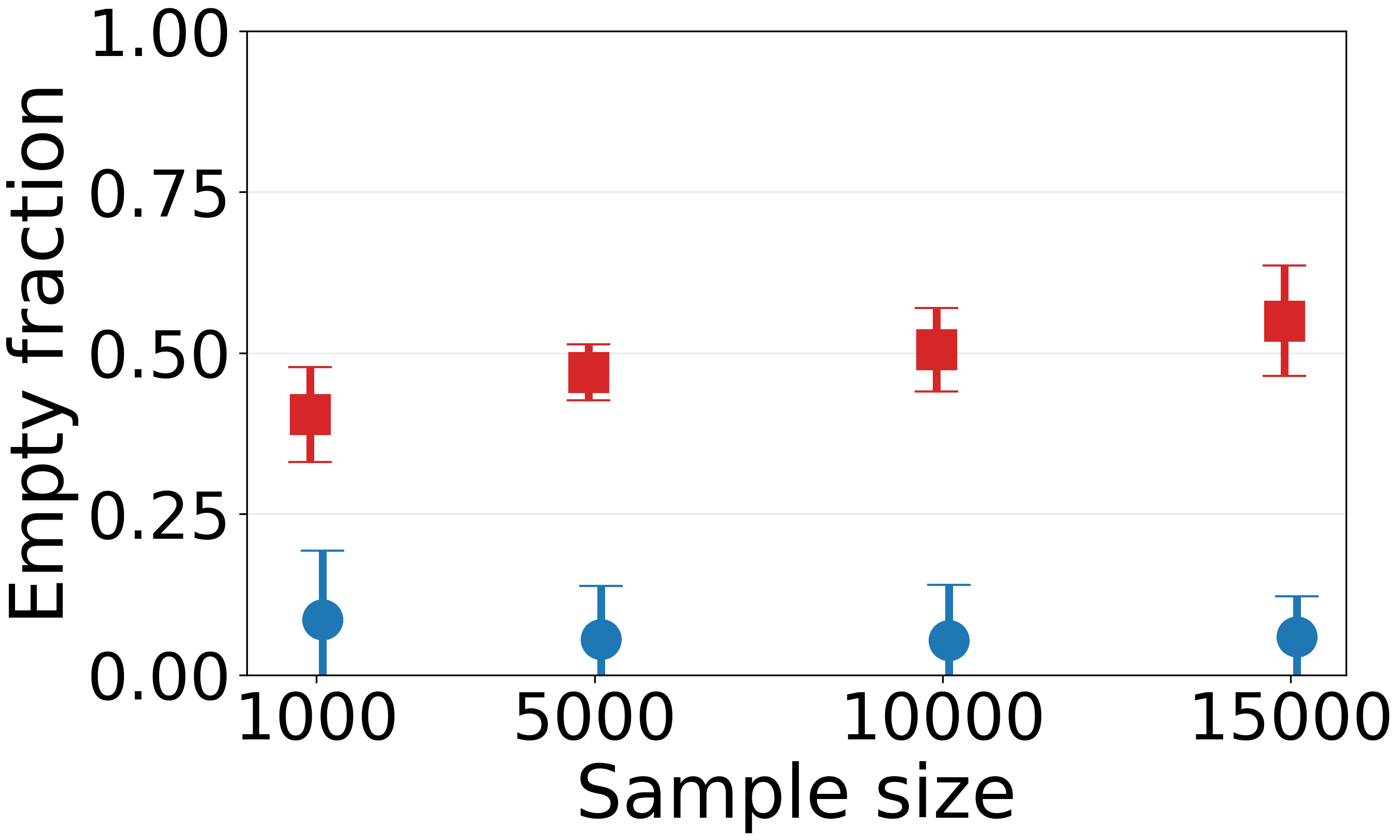}
        \caption{15 nodes}
    \end{subfigure}
    \begin{subfigure}{0.24\textwidth}
        \centering
        \includegraphics[width=\linewidth]{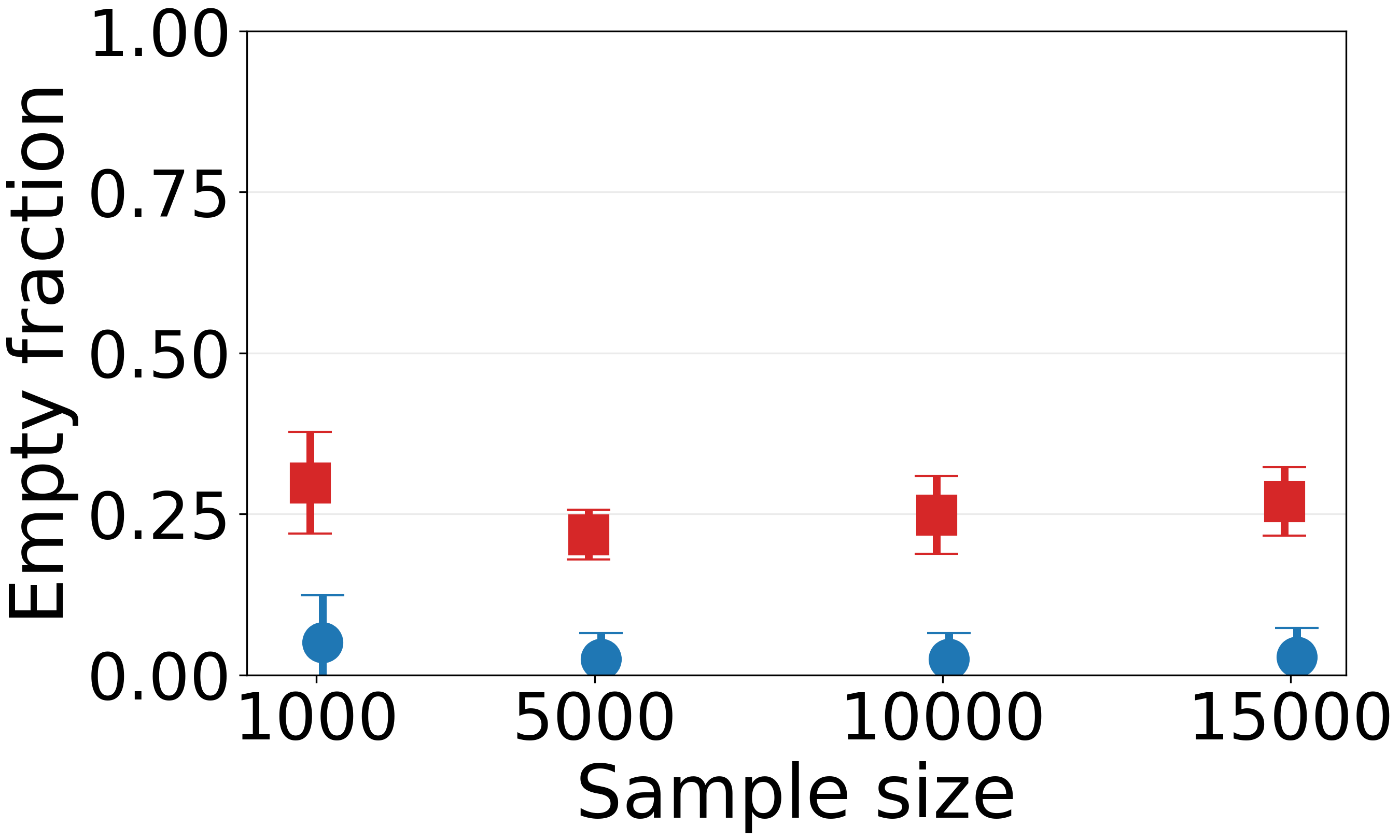}
        \caption{25 nodes}
    \end{subfigure}
    \begin{subfigure}{0.24\textwidth}
        \centering
        \includegraphics[width=\linewidth]{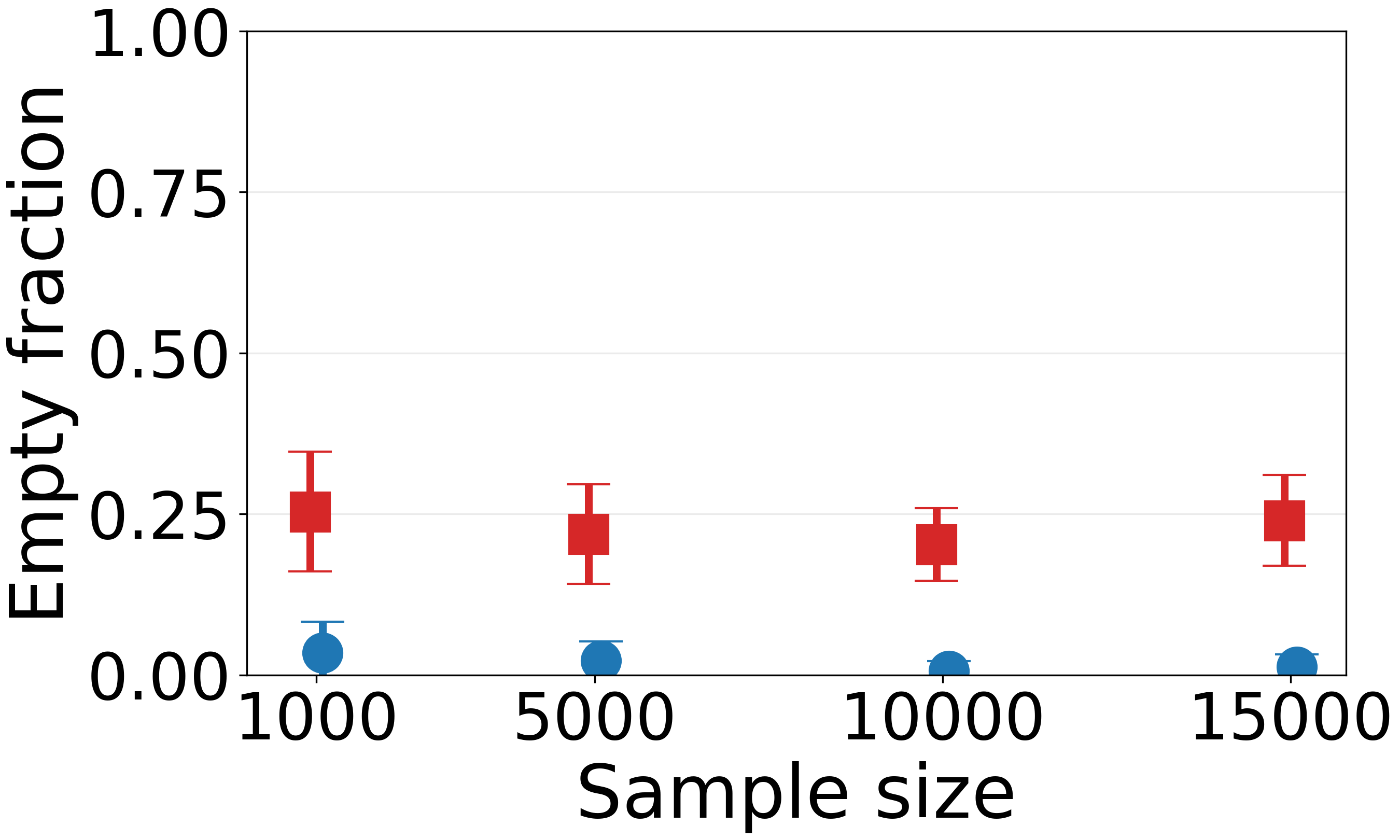}
        \caption{50 nodes}
    \end{subfigure}
    \caption{Empty fraction of linear and non-linear models in acyclic (red square) and cyclic (blue circle) settings across sample sizes. Markers show the mean, with error bars indicating one standard deviation. Subfigures correspond to graph sets with different numbers of nodes. }
    \label{fig:empty_frac_small_graph}
\end{figure}


\end{document}